\documentclass[11pt, a4paper]{article}
\usepackage{amsmath,amsfonts,amsthm,mathrsfs}

\pdfoutput=1

\numberwithin{equation}{section}
\newtheorem{theorem}{Theorem}
\newtheorem{lemma}{Lemma}
\newtheorem{proposition}{Proposition}

\newtheorem{definition}{Definition}
\newtheorem{remark}{Remark}
\newtheorem{example}{Example}
\usepackage[scale={0.72,0.743}, hmarginratio=1:1, vmarginratio=1:1, headheight=2ex, headsep = 2ex, footskip = 3.9ex, nomarginpar]{geometry}

\usepackage{booktabs}
\usepackage{diagbox} 
\usepackage{multirow} 
\usepackage{makecell} 
\usepackage{longtable} 
\usepackage{array}
\usepackage{tabularx}
\usepackage{caption}	
\usepackage{graphicx}
\usepackage{subfigure} 
\usepackage[numbers]{natbib}
\usepackage{algorithm} 
\usepackage{algorithmic} 

\def\NN{\mathbb N}

\def\RR{\mathbb R}

\begin{document}
	
	\title{Nonlinear functional regression by functional deep neural network with kernel embedding}
	\author{Zhongjie Shi\\ \footnotesize Department of Statistics and Actuarial Science, University of Hong Kong, \\ \footnotesize
		 Pok Fu Lam Road, Hong Kong, Email: zshi2@hku.hk  
	\and
	Jun Fan\\ \footnotesize Department of Mathematics, Hong Kong Baptist University, \\
	\footnotesize  Kowloon, Hong Kong, Email: junfan@hkbu.edu.hk
	\and
	Linhao Song\\ \footnotesize School of Mathematics and Statistics
	Central South University\\
	 \footnotesize Hunan, 410083, China, Email: songincsu@csu.edu.cn
	\and
	Ding-Xuan Zhou\\ \footnotesize School of Mathematics and Statistics, University of Sydney, \\
	\footnotesize Sydney, NSW 2006, Australia, Email: dingxuan.zhou@sydney.edu.au
	\and
	Johan A.K. Suykens \\ \footnotesize Department of Electrical Engineering, ESAT-STADIUS, KU Leuven, \\ \footnotesize
	Kasteelpark Arenberg 10, B-3001 Leuven, Belgium, Email: johan.suykens@esat.kuleuven.be}

	\date{}
	
	\maketitle
	
\begin{abstract}
	Recently, deep learning has been widely applied in functional data analysis (FDA) with notable empirical success. However, the infinite dimensionality of functional data necessitates an effective dimension reduction approach for functional learning tasks, particularly in nonlinear functional regression. In this paper, we introduce a functional deep neural network with an adaptive and discretization-invariant dimension reduction method. Our functional network architecture consists of three parts: first, a kernel embedding step that features an integral transformation with an adaptive smooth kernel; next, a projection step that utilizes eigenfunction bases based on a projection Mercer kernel for the dimension reduction; and finally, a deep ReLU neural network is employed for the prediction. Explicit rates of approximating nonlinear smooth functionals across various input function spaces by our proposed functional network are derived. Additionally, we conduct a generalization analysis for the empirical risk minimization (ERM) algorithm applied to our functional net, by employing a novel two-stage oracle inequality and the established functional approximation results. Ultimately, we conduct numerical experiments on both simulated and real datasets to demonstrate the effectiveness and benefits of our functional net.
	
	
\end{abstract}


\section{Introduction}
Functional data analysis (FDA) is a rising subject in recent scientific studies and daily life, which analyzes data with information about curves, surfaces, or anything else over a continuum, such as time, spatial location, and wavelength which are commonly considered in physical background \citep{ramsay2005,ferraty2006nonparametric}. 

{One significant task in FDA is functional regression, which aims to learn the relationship between a functional covariate and a scalar (or functional) response. Among the various regression models, the functional linear model and its variants, such as the generalized linear model, single-index model, and multiple-index model, are widely utilized and studied \citep{Yao2005, Muller2005, Chen2011}. While these models are simple and interpretable, they can struggle in scenarios where the true relationship is nonlinear. Given the impressive empirical success of deep learning in automatic feature extraction and flexible architecture design for various data formats, it is logical to explore the application of deep learning techniques in functional nonlinear regression.}

In functional regression, the data are obtained from a two-stage process. The first-stage dataset $\{f_i(\cdot), y_i\}_{i=1}^m$ are i.i.d. sampled from some true unknown probability distribution, where $f_i$ are random functions and $y_i$ are corresponding responses. However, in practice, we usually cannot observe $f_i$ over the entire continuum. Instead, we observe $f_i$ at discrete grid points $\{t_{i,j}\}_{j=1}^{n_i}$. Thus, the second-stage data $\{\{f_i(t_{i,j})\}_{j=1}^{n_i}, y_i\}_{i=1}^m$ is what we typically work with in practice. Since the functional data are intrinsically infinite-dimensional, one key idea of utilizing deep learning is to first summarize the information contained in each function $f_i$ into a finite-dimensional vector based on the discrete observations $\{f_i(t_{i,j})\}_{j=1}^{n_i}$, and then feed the resulting multivariate data into a deep neural network. 

{
	Recently, much effort has been dedicated to this idea in the literature. \cite{chen1995universal} proposed a functional network structure that directly uses the discrete observations \(\{f_i(t_{i,j})\}_{j=1}^{n}\) at grid points $t_{i,j}=t_j$, treating these observations as multivariate input data for a shallow network. This straightforward network architecture was demonstrated to be universal in \cite{chen1995universal} and was subsequently extended to a deeper version by \cite{lu2021learning}. However, this structure is mesh-dependent, meaning that it requires a fixed sampling grid $\{t_{j}\}_{j=1}^{n}$. Consequently, if this grid is altered, it would be computationally costly to retrain the model. A straightforward approach to designing mesh-independent (or discretization invariant) models is to utilize basis representations. For any input function $f$, the truncated basis expansion $\sum_{k=1}^{d_1} \langle f, \phi_k\rangle \phi_k$ serves as an approximation of $f$ by projecting it onto the subspace spanned by an orthonormal basis $\left\{\phi_k\right\}_{k=1}^{d_1}$. The coefficient vector $[\langle f,\phi_1\rangle,\cdots,\langle f,\phi_{d_1}\rangle)]^T$ can then be fed into various deep neural networks, such as Multi-Layer Perceptron (MLP) and Radial-Basis Function Networks (RBFN) \citep{rossi2005representation}. The universality and consistency of these models have been established in \cite{rossi2005functional}. The choice of basis for dimension reduction can vary. One approach is to use preselected bases that do not require learning, such as Legendre polynomials \citep{mhaskar1997neural}, B-splines \citep{rice2001nonparametric,cardot2003spline}, and Fourier basis \citep{kovachki2021universal}. However, these preselected bases do not fully leverage the information contained in the data. As a result, we often need a high-dimensional vector to reconstruct the original function, which can lead to the curse of dimensionality in the subsequent deep neural network. Therefore, it might be better to consider a data-dependent basis, such as the principal component basis \citep{besse1986principal,silverman1996smoothed,yao2005functional} or neural network basis \citep{rossi2002functional,yao2021deep}. However, these approaches come with their own drawbacks. While the principal component method effectively captures information from the input data, it overlooks the information derived from the output data. In the case of neural network basis, there are numerous free parameters within the basis layer that must be trained alongside the parameters of the subsequent network layers. This substantially increases the model's capacity, requiring a larger dataset to prevent the risk of overfitting.}

{
	Therefore, this paper aims to design a discretization-invariant dimension reduction method that adapts to both input and output data without increasing the complexity of the network structure. The key technique in our dimension reduction method is called kernel embedding, which is inspired by the kernel mean embedding approach in \cite{smola2007hilbert}, where a distribution is mapped to an element in a reproducing kernel Hilbert space (RKHS) through an embedding map. Our dimension reduction process can be summarized in two main steps. First, we specify an embedding kernel $K: \Omega \times \Omega \to \RR$ and apply the integral transformation $L_K f_i(x)=\int_{\Omega} K(x,t)f_i(t) dt$ induced by $K$ as the embedding map. The input functional data $f_i$ will then be transformed to $L_Kf_i$. Since only the second-stage data $\{f_i(t_{i,j})\}_{j=1}^{n_i}$ are accessible, this embedding can be approximated using an empirical alternative $\widehat{L_K}$, which serves as an estimate of $L_K$. Formal definition of $\widehat{L_K}$ will be provided in the subsequent section.  After the embedding step, we then map $\widehat{L_K}f$ to a $d_1$ dimensional coefficient vector, through a projection onto the subspace spanned by the first $d_1$ eigenfunctions of the integral operator induced by a projection Mercer kernel ${K_0}$. This coefficient vector is then fed into a deep neural network to predict the output. Since the embedding kernel $K$ is equipped with hyperparameters that are optimized through cross-validation, this approach offers adaptivity in the dimension reduction process without increasing the model's capacity.}

{
	In this paper, we propose a functional network structure that integrates the previously discussed dimension reduction method. Theoretical analysis and numerical experiments are conducted to evaluate this structure. In summary, our contributions are as follows:
	\begin{enumerate}
		\item We evaluate the expressivity of our proposed functional network by deriving explicit rates of approximating nonlinear smooth functionals defined on various input function spaces. For input functions within Besov spaces or mixed smooth Sobolev spaces, we establish logarithmic rates of approximation. For input functions within Gaussian RKHSs, we enhance the approximation rates to $\exp(-\alpha(\log M)^{\beta})$, where $\alpha > 0$ and $0 < \beta < 1$. This improvement in the rates indicates that our functional network can exploit the regularity within the input functions.
		\item We propose a learning algorithm through empirical risk minimization (ERM) applied to our functional network based on the second-stage data. Generalization analysis is carried out on this learning algorithm within the classical learning theory framework. Specifically, we establish a new two-stage oracle inequality that considers both the first-stage sample size $m$ and the second-stage sample size $n$. By applying this new oracle inequality along with a theoretically optimal quadrature scheme, we derive convergence rates for learning target functionals defined on various input function spaces. Our theoretical findings reveal that the second-stage sample size required for achieving unimpaired generalization error can be significantly smaller than that of the first-stage sample size, indicating the potential of our functional network in handling sparsely observed functional data.
		\item Numerical experiments conducted on both simulated and real datasets indicate that our functional network utilizing kernel embedding surpasses the performance of functional networks that rely on other baseline dimension reduction techniques, including discrete observations, B-splines, FPCA, and neural network bases. The numerical results also provide valuable insights into the key factors affecting the generalization performance of our functional network. Furthermore, we validate the advantages of our approach by showcasing its discretization-invariant properties, its adaptability to a range of datasets, and its robustness to noisy observations.
	\end{enumerate}
}

The rest of the paper is arranged as follows. In Section 2, we introduce the architecture
of our functional deep neural network with kernel embedding and elaborate on its practical implementation. In Section 3,  we conduct the theoretical analysis of our functional network by presenting its approximation rates for various input spaces when the target functional is smooth. In Section 4, we carry out a generalization analysis by first providing a two-stage oracle inequality, and then utilizing it to derive learning rates for the input function spaces considered in Section 3. {Section 5 discusses related work and summarizes the theoretical findings of this paper.} In Section 6, numerical experiments are performed to verify the performance and benefits of our proposed functional network. The proofs of the main results in this paper are given in the Appendix.

\section{Architecture of Functional Deep Neural Network with Kernel Embedding}
\label{section1}
In this section, we give a detailed introduction to the architecture of our functional deep neural network with the usage of kernel embedding. Suppose that the input function space $\mathcal{F}$ is a compact subset of $L_\infty(\Omega)$ and of $L_2(\Omega)$, where $\Omega$ is a measurable subset of $\RR^d$. We denote $L_p(\Omega)$ as the Lebsegue space of order $p$ with respect to the Lebsegue measure on $\Omega$, and denote its norm as $\|\cdot\|_{L_p(\Omega)}$. Furthermore, we denote $L_p(\mu)$ as the Lebesgue space of order $p$ with respect to the measure $\mu$ on $\Omega$, and denote its norm as $\|\cdot\|_{L_p(\mu)}$. 

{
	Let $\mu$ be a positive Borel measure on $\Omega$ and $L_2(\mu)$ be the Hilbert space of the square integrable functions on $\Omega$ w.r.t.\ $\mu$. For a \emph{kernel} $K: \Omega \times \Omega \to \RR$, we denote $L_K^\mu$ as
	\begin{equation} \label{defintop}
		\left(L_K^\mu f\right)(x)= \int_{\Omega} K(x,t)f(t)d\mu(t).
	\end{equation}
	Moreover, if $\Omega$ is a compact subset of $\RR^d$ and the kernel $K$ qualifies as a \emph{Mercer kernel} (or \emph{reproducing kernel}), meaning that it is continuous, symmetric, and positive definite. In this case, the linear \textit{integral operator} $L_K^\mu: L_2(\mu) \to C(\Omega)$ is a compact, self-adjoint, positive operator. Consequently, the Spectral Theorem is applicable, indicating that there exists an orthonormal basis $\{\phi_1,\phi_2,\cdots\}$ of $L_2(\mu)$ consisting of the eigenfunctions of $L_K^\mu$. If $\lambda_k$ denotes the eigenvalue associated with $\phi_k$, the set $\{\lambda_k\}$ is either finite or approaches zero when $k \to \infty$. Moreover, Mercer's Theorem states that uniformly,
	\begin{equation}
		K(x,t)= \sum_{k=1}^\infty \lambda_k \phi_k (x) \phi_k (t).
	\end{equation} 
	Furthermore, if we denote
	\begin{equation}\label{defkerbound}
		C_K= \sup_{x,t \in \Omega} \lvert K(x,t)\rvert,
	\end{equation}
	then the operator norm is bounded as $\Vert L_K^\mu \Vert \leq \sqrt{\mu(\Omega)}C_K$.
	We note that in this paper, unless specified otherwise, the term ``kernel" refers specifically to a Mercer kernel.
}
\begin{definition} [Functional net with kernel embedding] \label{deffunctionalNN}
	{
		Let $f\in \mathcal{F}$ represent an input function. We start by performing a kernel embedding step on $f$ through an integral transformation, given by:
		\begin{equation} \label{kernelembedding}
			L_K f= \int_{\Omega} K(\cdot,t)f(t) dt,
		\end{equation}
		where the embedding kernel $K: \Omega \times \Omega \to \RR$ is a continuous kernel.
	}
	
	{
		Next, we define the functional deep neural network with a depth of $J$ and width $\{d_j\}_{j=1}^J$. This network is structured iteratively as follows:
		\begin{equation}
			h^{(j)}(f)= \begin{cases}
				T(L_K f), & j=1, \\
				\sigma\left(F^{(j)} h^{(j-1)}(f)+ b^{(j)}\right), & j=2,3,\dots,J.
			\end{cases}
		\end{equation}
		In this formulation, the first layer is identified as a projection step using the basis $\{\phi_i\}_{i=1}^{d_1}$:
		\begin{equation} \label{coeffvector}
			T(L_K f)
			=  \left[\int_{\Omega} L_K f(t) \phi_1(t) d\mu(t), \int_{\Omega} L_K f(t) \phi_2(t) d\mu(t), \cdots, \int_{\Omega} L_K f(t) \phi_{d_1}(t) d\mu(t)  \right]^T,
		\end{equation}
		where $\{\phi_i, \lambda_i\}$ represents the eigensystem of the integral operator $L_{K_0}^\mu$ induced by a projection Mercer kernel $K_0$ and a positive Borel measure $\mu$.
		The subsequent layers follow the standard deep neural network, where $\{F^{(j)} \in \RR^{d_j \times d_{j-1}}\}_{j=1}^J$ are the weight matrices, $\{b^{(j)} \in \RR^{d_j} \}_{j=1}^J$ are the bias vectors, and $\sigma(u)=\max \{u,0\}$ is the ReLU activation function.}
	
	{
		The final output of the functional net is derived as a linear combination of the last layer
		\begin{equation}
			F_{NN}(f)=c\cdot h^{(J)}(f),
		\end{equation}
		where $c\in \RR^{d_{J}}$ is the coefficient vector.
		Figure \ref{new structure} specifically depicts the architecture of the functional net with kernel embedding.}
	\begin{figure}[t]
		{
			\centering
			\includegraphics[width=0.8\textwidth]{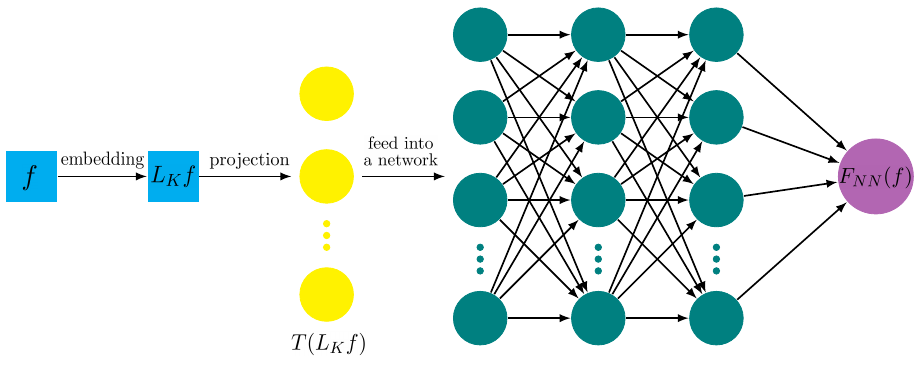}
			\caption{{Architecture of functional net with kernel embedding.}}
			\label{new structure}
		}
	\end{figure}
\end{definition}

{
	Below, we provide two examples regarding the selection of embedding kernels and projection kernels in our functional network. In the first example, we choose the embedding kernel $K$ to be the product of a Gaussian kernel with a density function. In the second example, the embedding kernel $K$ is selected as a linear combination of Gaussian kernels. Let $\mathcal{H}_\gamma$ denote the RKHS corresponding to the Gaussian kernel $k_\gamma$:
	\begin{equation}
		k_\gamma(x,t)=\exp\left(-\frac{\|x-t\|_2^2}{\gamma^2}\right), \qquad x,t\in \RR^d,
	\end{equation}
	where $\gamma > 0$ represents the bandwidth, and we denote its norm as $\|\cdot\|_{H_\gamma}$.
}
\begin{example}
	If we choose the embedding kernel $K$ as $K(x,t)= k_\gamma(x,t) u(t)$ where $u$ is the density of a positive Borel measure $\mu$, notice that $L_K f \in \mathcal{H}_\gamma(\Omega)$ for any $f\in L_2(\Omega)$, we can set the projection kernel $K_0$ to be $k_\gamma$.
\end{example}

\begin{example} \label{example2}
	If we choose the embedding kernel $K$ to be a linear combination of a collection of Gaussian kernels with varying bandwidths $\gamma_i$:
	\begin{equation}
		K(x,t)=\sum_{i=1}^P \beta_i k_{\gamma_i}(x,t).
	\end{equation}
	Denote $\gamma=  \min_i \gamma_i$. Notice that $L_K f \in \mathcal{H}_\gamma(\Omega)$ for any $f\in L_2(\Omega)$, we can set the projection kernel $K_0$ to be $k_\gamma$.
\end{example}

{
	The architecture described in Definition \ref{deffunctionalNN} serves as a theoretical framework. Next, we will outline how this framework can be practically implemented. The two critical steps for empirical implementation are the kernel embedding step specified in (\ref{kernelembedding}) and the projection step described in (\ref{coeffvector}).
}

{
	For the kernel embedding step defined in  \eqref{kernelembedding}, we can assume, without loss of generality, that the domain of the input functions is $\Omega=[0,1]^d$. This allows us to express the embedding step as follows:
	\begin{equation} \label{kernelquadrature}
		L_K f_i= \int_\Omega K(\cdot, t) f_i(t) d \mathcal{U}(t),
	\end{equation}
	where $d\mathcal{U}$ represents the uniform distribution on $\Omega$. The quadrature problem focuses on approximating this integral using linear combinations:
	\begin{equation} \label{quadraturescheme}
		\widehat{L_K} f_i =  \sum_{j=1}^{n_i} \theta_{i,j} K(\cdot,t_{i,j}),
	\end{equation}   
	where the points $t_{i,j} \in \Omega$ and the weights $\theta_{i,j}$ are chosen appropriately to minimize the approximation error as much as possible.
}

\begin{remark}
	{
		The standard Monte-Carlo method is to consider observation points $\{t_{i,j}\}_{j=1}^{n_i}$ i.i.d. sampled from $d\mathcal U$ and the weights $\theta_{i,j}= f_i(t_{i,j})/n_i$, which results in a decrease of the error in the rate of $1/\sqrt{n_i}$. Alternative sampling methods, such as Quasi Monte-Carlo methods and Trapezoidal rules, can lead to improved convergence rates.
	}
\end{remark}


\begin{algorithm} [t]
	\caption{{Training functional network}} 
	\label{algorithm}
	{
		{\bf Input:}
		second-stage data $\widehat{D}=\{\{t_{i,j},f_i(t_{i,j})\}_{j=1}^{n_i},y_i\}_{i=1}^m$, a 3D grid of hyperparameters: Gaussian kernel bandwidth $\gamma > 0$, scaling parameter $\beta>0$ of the Gaussian distribution $\mu$, reduced dimension $d_1$.  \\
		{\bf Output:}
		the learned functional net model $F_{\widehat{D}}$, and the corresponding test-set MSE.
		\begin{algorithmic}
			\STATE \textbf{Step 1:} For the particular hyperparameters $\gamma, \beta, d_1$, scale the discretization points $\{t_{i,j}\}_{j=1}^{n_i}$ to the domain $\Omega=[0,1]^d$, convert the input data $\{t_{i,j},f_i(t_{i,j})\}_{j=1}^{n_i}$ from $\widehat{D}$ into a $d_1$ dimensional vector $x_i \in \RR^{d_1}$ for $i=1, \dots, m$, using the following calculation: \\
			\begin{equation*}
				(x_i)_\ell = 	\widehat{T} (\widehat{L_K} f_i)_\ell= \sum_{k=1}^{N} w_k \sum_{j=1}^{n_i} \theta_{i,j} K(s_k,t_{i,j}) \phi_\ell(s_k), \quad \ell=1,\dots,d_1.
			\end{equation*}
			\STATE \textbf{Step 2:} Split the pre-processed data $\{x_i,y_i\}_{i=1}^m$ to the training, validation, and test sets.
			\STATE \textbf{Step 3:} Use a training set to train a functional net model denoted as $F_{\gamma,\beta,d_1}$.
			\STATE \textbf{Step 4:} Employ cross-validation to determine the optimal model $F_{\widehat{D}}$ across the 3D hyperparameter grid based on the validation set, and compute the mean squared error (MSE) on the test set using $F_{\widehat{D}}$.
		\end{algorithmic}
	}
\end{algorithm}

{
	For the projection step defined in \eqref{coeffvector}, we apply numerical integration over the discrete grid points $\{s_k\}_{k=1}^{N}$ within $\Omega$, using weights $\{w_k\}_{k=1}^N$. This results in the following equation:
	\begin{equation} \label{numericalke}
		\widehat{T} (\widehat{L_K} f_i)_\ell= \sum_{k=1}^{N} w_k \sum_{j=1}^{n_i} \theta_{i,j} K(s_k,t_{i,j}) \phi_\ell(s_k), \quad \ell=1,\dots,d_1.
	\end{equation}
	When the grid points are chosen to be sufficiently dense with a large value of $N$, this numerical integration can yield an accurate approximation of the integral.
}

{
	Finally, we provide Algorithm \ref{algorithm} as a comprehensive description of how to train our functional network in practice. This algorithm is precisely what we used for the training in our numerical simulations. In this algorithm, we select the embedding kernel $K$ to be the Gaussian kernel $k_\gamma$ with bandwidth $\gamma > 0$. The basis functions $\{\phi_\ell\}_{\ell=1}^{d_1}$ used in the projection step are chosen as the eigenfunctions of the integral operator $L_K^\mu$, which are explicitly represented in terms of Hermite polynomials, as discussed in \cite[pp. 28]{Fasshauer2011}. Additionally, the measure $d\mu$ is selected to be the Gaussian distribution, characterized by the density function $u(x)= \frac{\beta}{\sqrt{\pi}} \exp^{-\beta^2 x^2}$, where $\beta$ serves as a global scaling parameter.
}

\section{Approximation Results} \label{section2}
In this section, we state the main approximation results of our proposed functional net with kernel embedding. We begin by deriving rates of approximating nonlinear smooth functionals defined on Besov spaces. The utilization of data-dependent kernels allows us to attain improved approximation rates, particularly when the input function space is smaller, such as in the case of Gaussian RKHSs and mixed-smooth Sobolev spaces. These findings highlight the flexibility of our functional network in exploiting the regularity properties of the input functions.

\subsection{Rates of approximating nonlinear smooth functionals on Besov spaces}
\label{section21}
Our primary finding focuses on the rates of approximating nonlinear smooth functionals in Besov spaces $B_{2,\infty}^\alpha(\Omega)$ with $\alpha > 0$ using our functional net. Besov spaces provide a more nuanced measure of smoothness compared to Sobolev spaces $W_2^\alpha(\Omega)$, which consists of functions whose weak derivatives up to order $\alpha$ exist and belong to $L_2(\Omega)$. For more information on Sobolev and Besov spaces, refer to \cite[Chapter 2]{devore1993constructive}.

Specifically, for $r\in \NN$, the $r$-th difference $\Delta_h^r(f,\cdot): \Omega \to \RR$ for a function $f\in L_2(\Omega)$ and $h=(h_1,\dots,h_d)\in [0,\infty)^d$ is defined as
\begin{equation}
	\Delta_h^r(f,x)= 
	\begin{cases}
		\sum_{j=0}^r  \binom{r}{j}  (-1)^{r-j}f(x+jh), & \hbox{if} \ x\in X_{r,h}, \\
		0 & \hbox{if} \ x\notin X_{r,h},
	\end{cases}
\end{equation}
where $X_{r,h}:= \{x\in \Omega: x+sh \in \Omega, \forall s\in [0,r]\}$. To measure the smoothness of functions, the $r$-th modulus of smoothness for  $f\in L_2(\Omega)$ is defined as
\begin{equation}
	\omega_{r,L_2(\Omega)}(f,t)= \sup_{\|h\|_2 \leq t} \left\| \Delta_h^r(f,\cdot)\right\|_{L_2(\Omega)}, \qquad \ t\geq 0.
\end{equation}
Let $r=\lfloor\alpha\rfloor+1$, where $\lfloor \alpha\rfloor$ is the greatest integer that is smaller than or equal to $\alpha$. Then the Besov space $B_{2,\infty}^\alpha(\Omega)$ is defined as
\begin{equation}
	B_{2,\infty}^\alpha(\Omega)= \left\{f\in L_2(\Omega): |f|_{B_{2,\infty}^\alpha(\Omega)}< \infty \right\},
\end{equation}
where $|f|_{B_{2,\infty}^\alpha(\Omega)}:= \sup_{t>0} (t^{-\alpha} \omega_{r,L_2(\Omega)}(f,t))$ is the semi-norm of $B_{2,\infty}^\alpha(\Omega)$, and the norm of $B_{2,\infty}^\alpha(\Omega)$ is defined as $\|f\|_{B_{2,\infty}^\alpha(\Omega)}= \|f\|_{L_2(\Omega)} + |f|_{B_{2,\infty}^\alpha(\Omega)}$.

Suppose that the target functional $F: \mathcal{F} \to \RR$ is continuous with the modulus of continuity defined as
\begin{equation} 
	\omega_F(r)= \sup\{|F(f_1)-F(f_2)|:  f_1,f_2 \in \mathcal{F}, \|f_1-f_2\|_{L_2(\mu)} \leq r\},
\end{equation} 
satisfying the condition that
\begin{equation} \label{functionalsmoothness}
	\omega_F(r)\leq C_F r^\lambda, \qquad \hbox{for some} \ \lambda\in (0,1],
\end{equation}
where the measure $\mu$ is a positive Borel measure, such as the Lebesgue measure or Gaussian measures restricted on $\Omega$, and $C_F$ is a positive constant.

Let us consider the input function domain as $\Omega=\RR^d$. Suppose that the input function space $\mathcal{F}$ is a compact subset of $L_\infty(\RR^d) \cap B_{2,\infty}^\alpha(\RR^d)$, where for any function $f\in \mathcal{F}$, its norm satisfies $\|f\|_{B_{2,\infty}^\alpha(\RR^d) }\leq 1$. In the following theorem, we demonstrate that our functional net can effectively approximate nonlinear smooth target functionals defined on Besov spaces. This result is obtained by selecting the embedding kernel $K$ as a linear combination of Gaussian kernels, as demonstrated in Example \ref{example2}. In particular, we have:
\begin{equation} \label{embedkernel1}
	K(x,t)= \sum_{j=1}^r \binom{r}{j} (-1)^{1-j} \frac{1}{j^d} \left(\frac{1}{\gamma^2 \pi} \right)^{\frac{d}{2}} k_{j\gamma} (x,t),
\end{equation}
where $r=\lfloor\alpha\rfloor+1$ is determined by the smoothness of the input function. As shown in Example \ref{example2}, $L_K f$ belongs to $H_\gamma(\RR^d)$. Therefore, we can choose the projection kernel $K_0$ as $k_\gamma$ in the projection step, and select the projection basis as the first $d_1$ eigenfunctions of the integral operator $L^\mu_{K_0}$. {We note that in this theorem, although the domain $\Omega= \RR^d$ is not compact, the selection of the Gaussian measure $\mu$ in $L_K^\mu$ guarantees that it remains a compact operator. Consequently, this allows the application of the Spectral Theorem, as noted in \citep[pp. 28]{Fasshauer2011}.}

\begin{theorem} \label{approxrate1}
	Let $\alpha>0$, $d,M \in \NN$. Assume that the input function space $\mathcal{F}$ is a compact subset of $L_\infty(\RR^d) \cap B_{2,\infty}^\alpha(\RR^d)$, with the condition that $\|f\|_{B_{2,\infty}^\alpha(\RR^d) }\leq 1$ for any function $f\in \mathcal{F}$. Additionally, suppose that the modulus of continuity of the target functional $F: \mathcal{F} \to \mathbb{R}$ adheres to condition (\ref{functionalsmoothness}) with $\lambda \in (0,1]$. By selecting the embedding kernel $K$ as defined in \eqref{embedkernel1}, the projection kernel as $K_0=k_\gamma$, and
	\begin{equation}
		d_1 = \tilde c_1 \frac{\log M}{\log \log M}, \quad \gamma = \tilde c_2 \left(\frac{\log \log M}{\log M}\right)^{\frac{1}{d}} \log\log M,
	\end{equation}
	there exists a functional network $F_{NN}$ that follows the architecture specified in Definition \ref{deffunctionalNN}, with $M$ nonzero parameters and the depth
	\begin{equation}
		J \leq \tilde c_3 \left( \frac{\log M}{\log\log M}\right)^2,
	\end{equation}
	such that
	\begin{equation}
		\sup_{f\in \mathcal{F}} |F(f)-F_{NN}(f)| \leq  \tilde c_4 \left(\log M\right)^{-\frac{\alpha\lambda}{d}}  \left(\log\log M\right)^{\left(\frac{1}{d}+1\right) \alpha\lambda},
	\end{equation}
	where $\tilde{c}_1, \tilde{c}_2, \tilde c_3, \tilde c_4$ are positive constants.
\end{theorem}

Since $W_2^\alpha(\mathbb{R}^d) \subset B_{2,\infty}^\alpha(\mathbb{R}^d)$, our findings are applicable to approximate nonlinear smooth functionals defined on Sobolev spaces as well. The leading term in the approximation rates aligns with previous studies \citep{mhaskar1997neural, song2022, songapproximation2023}, although the $\log \log M$ term exhibits slight variations. The polynomial rates in terms of $\log M$, as opposed to $M$ in traditional function approximation, arise from the curse of dimensionality, given that the input function space is infinite-dimensional. To address this curse of dimensionality, it is necessary to impose stronger smoothness conditions on the target functional. 

It is important to note that Theorem \ref{approxrate1} is specifically applicable to input function spaces defined as Besov spaces on $\mathbb{R}^d$, rather than on any subset of it. In the following, we will explore how to extend the results in Theorem \ref{approxrate1} to Sobolev spaces defined on a domain $\Omega \subset \mathbb{R}^d$.  Let us denote $H_0^\alpha(\Omega)$ as the closure of infinitely differentiable compactly supported functions $C_c^\infty(\Omega)$ within the space $W_2^\alpha(\Omega)$. For any function $f \in H_0^\alpha(\Omega)$, we can consider its extension to $\mathbb{R}^d$ by defining $\tilde{f} \in L_2(\mathbb{R}^d)$ as follows:
\begin{equation}
	\tilde{f}(x)= \begin{cases}
		f(x) & x\in \Omega, \\
		0 & \hbox{otherwise}.
	\end{cases}
\end{equation}
With this construction, it follows that $\tilde{f} \in W_2^\alpha(\mathbb{R}^d)$, and we have the equality for the norms: $\|\tilde f\|_{W_2^\alpha(\RR^d)}= \|f\|_{W_2^\alpha(\Omega)}$ \citep[Lemma 3.22]{adams2003sobolev}. This enables us to obtain the subsequent result.

\begin{theorem} \label{approxrate12}
	Let $\alpha, d,M \in \NN$, $\Omega=[0,1]^d$. Assume that the input function space $\mathcal{F}$ is a compact subset of $H_0^\alpha(\Omega)$, with the condition that  $\|f\|_{W_{2}^\alpha(\Omega) }\leq 1$ for any function $f\in \mathcal{F}$. Additionally, suppose that the modulus of continuity of the target functional $F: \mathcal{F} \to \RR$ adheres to condition (\ref{functionalsmoothness}) with $\lambda\in (0,1]$. By selecting the embedding kernel $K$ as defined in \eqref{embedkernel1}, the projection kernel as $K_0=k_\gamma$, and
	\begin{equation} \label{hyper1}
		d_1 = \tilde c_1 \frac{\log M}{\log \log M}, \quad \gamma = \tilde c_2 \left(\frac{\log \log M}{\log M}\right)^{\frac{1}{d}} \log\log M,
	\end{equation}
	there exists a functional network $F_{NN}$ that follows the architecture specified in Definition \ref{deffunctionalNN}, with $M$ nonzero parameters and the depth
	\begin{equation}
		J \leq \tilde c_3 \left( \frac{\log M}{\log\log M}\right)^2,
	\end{equation}
	such that
	\begin{equation}
		\sup_{f\in \mathcal{F}} |F(f)-F_{NN}(f)| \leq  \tilde c_4 \left(\log M\right)^{-\frac{\alpha\lambda}{d}}  \left(\log\log M\right)^{\left(\frac{1}{d}+1\right) \alpha\lambda},
	\end{equation}
	where $\tilde c_1, \tilde c_2, \tilde c_3, \tilde c_4$ are positive constants.
\end{theorem}

It would be intriguing to explore additional, more general scenarios in which the approximation results presented in Theorem \ref{approxrate1} can be extended to Besov spaces defined on domains $\Omega \subset \mathbb{R}^d$. Such investigations could enrich our understanding of functional approximation in a broader context and potentially lead to new insights and results applicable to a wider range of applications.

\subsection{Rates of approximating nonlinear smooth functionals on Gaussian RKHSs}
\label{section22}

We then investigate scenarios where the input function spaces are more limited. For example, when the input function spaces are defined as RKHSs induced by specific Mercer kernels, our functional net can still attain satisfactory approximation rates. This is made possible by the flexible choice of the embedding kernel $K$.

Let us consider $\Omega = [0, 1]^d$ and assume that the input function space $\mathcal{F}$ is a compact subset of the unit ball of the Gaussian RKHS $H_\gamma(\Omega)$, which is indeed a subset of the Besov space $B_{2,\infty}^\alpha(\Omega)$ \citep{steinwart2008support}. Furthermore, we assume that the target functional $F: \mathcal{F} \to \mathbb{R}$ meets the same modulus of continuity condition as in (\ref{functionalsmoothness}). Our second main result demonstrates rates of approximating nonlinear smooth functionals on Gaussian RKHSs using our functional net. In this case, the projection kernel $K_0$ is chosen to be the Gaussian kernel $k_\gamma$, and the bases used in the projection step consist of the first $d_1$ eigenfunctions of the integral operator $L^\mu_{K_0}$. The embedding kernel $K$ is chosen as
\begin{equation} \label{embeddingkernelg}
	K(x,t)= k_\gamma (x,t) u(t),
\end{equation}
where $u$ is the density of the measure $\mu$.

\begin{theorem} \label{approxrate2}
	Let $\gamma>0$, $d,M \in \NN$, $\Omega=[0,1]^d$. Assume that the input function space $\mathcal{F}$ is a compact subset of the unit ball of Gaussian RKHS $H_\gamma (\Omega)$, and the modulus of continuity of the target functional $F: \mathcal{F} \to \RR$ satisfies the condition (\ref{functionalsmoothness}) with $\lambda\in (0,1]$. By selecting the embedding kernel $K$ as defined in \eqref{embeddingkernelg}, the projection kernel $K_0$ as $k_\gamma$, and
	\begin{equation} \label{hyper2}
		d_1 = c_6 \left(\log M \right)^{\frac{d}{d+1}},
	\end{equation}
	there exists a functional network $F_{NN}$ that follows the architecture specified in Definition \ref{deffunctionalNN}, with $M$ nonzero parameters and the depth
	\begin{equation}
		J \leq \tilde c_5 \left( \log M\right)^{\frac{2d}{d+1}},
	\end{equation}
	such that
	\begin{equation}
		\sup_{f\in \mathcal{F}} |F(f)-F_{NN}(f)| \leq \tilde c_6 e^{-c_7 \lambda \left(\log M \right)^{\frac{1}{d+1}}} \left(\log M\right)^{\frac{d}{d+1}},
	\end{equation}
	where $\tilde c_5, \tilde{c}_6, c_6, c_7$ are positive constants, with $c_7> \left(\frac{d}{3e}\right)^{\frac{d}{d+1}}$.
\end{theorem}

It is essential to highlight that although the proof technique employed for this result may seem more straightforward than that used in Theorem \ref{approxrate1}, it cannot be directly applied to prove Theorem \ref{approxrate1}. This restriction occurs because this proof technique is applicable only when Sobolev spaces $W_2^\alpha(\Omega)$ are RKHSs, which is true if and only if $\alpha > \frac{d}{2}$ \citep[Theorem 121]{berlinet2011reproducing}.

In this context, the dominant term of the approximation rate is given by $e^{-c_7 \lambda \left(\log M \right)^{\frac{1}{d+1}}}$, where  $c_7> \left(\frac{d}{3e}\right)^{\frac{d}{d+1}}$, along with an additional $\log M$ factor. This rate surpasses the $(\log M)^{-a}$ bound for any polynomial rate of $\log M$ with $a > 0$, which characterizes the situation in Theorem \ref{approxrate1}. However, it is inferior to the $M^{-a}$ rate for any polynomial rate of $M$ with $a > 0$ in the asymptotic sense when $M \to \infty$. This suggests that even when the input function spaces are infinitely differentiable, we still cannot achieve polynomial approximation rates due to the curse of dimensionality.

\subsection{Rates of approximating nonlinear smooth functionals on mixed smooth Sobolev spaces}
\label{section23}

It is important to note that although the approximation rates for target functionals defined on Gaussian RKHSs in Section \ref{section22} show significant improvement compared to those on Besov spaces in Section \ref{section21}, the dominant terms of these rates still depend on the dimension $d$ of the input function spaces. Our third main result demonstrates that our functional net can achieve approximation rates that are independent of $d$ when the input function spaces possess certain special properties, specifically in the case of mixed smooth Sobolev spaces.

Let $\partial^{(k)}$ denote the $k$-th derivative for a multi-index $k \in \NN_0^d$. For an integer $\alpha > 0$, the mixed smooth Sobolev spaces $H_{mix}^\alpha(\Omega)$, which consist of functions with square-integrable partial derivatives with all individual orders less than $\alpha$, are defined as:
\begin{equation}
	H_{mix}^\alpha(\Omega)= \left\{ f\in L_2(\Omega): \partial^{(k)} f \in L_2(\Omega), \forall  \| k\|_\infty \leq \alpha \right\},
\end{equation}
where $\|k\|_\infty= \max_{1\leq i\leq d} k_i$. The norm on this space is given by:
\begin{equation}
	\Vert f\Vert_{H_{mix}^\alpha(\Omega)}= \left( \sum_{\|k\|_\infty \leq \alpha} \Vert  \partial^{(k)} f \Vert_{L_2(\Omega)}^2 \right)^{\frac{1}{2}}.
\end{equation}

The mixed smooth Sobolev space can be understood as a tensor product of univariate Sobolev spaces. Specifically, if we denote
\begin{equation}
	k_\nu(x,y;\ell)=\frac{2^{1-\nu}}{\Gamma(\nu)} \left( \sqrt{2\nu} \ell |x-y| \right)^\nu K_\nu\left(\sqrt{2\nu} \ell |x-y| \right)
\end{equation}
as the Mat\'ern kernel that reproduces the univariate Sobolev space $H^\alpha([0,1])$ with $\alpha = \nu + \frac{1}{2}$ (where $\Gamma$ denotes the gamma function, $K_\nu$ is the modified Bessel function of the second kind, and $\ell$ is a positive constant), the eigenvalues of the integral operator associated with this kernel decay polynomially as $\lambda_k = O(k^{-2\alpha})$ \citep{wendland2004scattered}. Consequently, the mixed smooth Sobolev space $H_{mix}^\alpha([0,1]^d)$ can be characterized as an RKHS induced by $K_\alpha$, which is the pointwise product of the individual Mat\'ern kernels, specifically:
$$K_\alpha(x,y) = \prod_{j=1}^{d} k_\nu(x_j,y_j), \quad \text{for } x,y \in [0,1]^d.$$

Assume that the input function space $\mathcal{F}$ is a compact subset of the unit ball of $H_{mix}^\alpha(\Omega)$, and that the target functional $F: \mathcal{F} \to \mathbb{R}$ adheres to the same modulus of continuity condition as in (\ref{functionalsmoothness}).  Our third main result demonstrates the rates of approximating this target functional by our functional net, where we select the projection kernel $K_0$ as $K_\alpha$, and the projection bases are chosen as the first $d_1$ eigenfunctions of the integral operator $L^\mu_{K_0}$. The embedding kernel $K$ is chosen as
\begin{equation} \label{embeddingkernelm}
	K(x,t)= K_\alpha (x,t) u(t),
\end{equation}
where $u$ is the density of the measure $\mu$.

\begin{theorem} \label{approxrate3}
	Let $\alpha,d,M \in \NN$, $\Omega=[0,1]^d$. Assume that the input function space $\mathcal{F}$ is a compact subset of the unit ball of the mixed smooth Sobolev space $H_{mix}^\alpha(\Omega)$, and the modulus of continuity of the target functional $F: \mathcal{F} \to \RR$ adheres to the condition (\ref{functionalsmoothness}) with $\lambda\in (0,1]$.  By selecting the embedding kernel $K$ as defined in \eqref{embeddingkernelm}, the projection kernel $K_0$ as $K_\alpha$, and
	\begin{equation} \label{hyper3}
		d_1 = C_4 \frac{\log M}{\log\log M},
	\end{equation}
	there exists a functional network $F_{NN}$ that follows the architecture specified in Definition \ref{deffunctionalNN}, with $M$ nonzero parameters and the depth
	\begin{equation}
		J \leq C_5 \left( \frac{\log M}{\log\log M}\right)^2,
	\end{equation}
	such that
	\begin{equation}
		\sup_{f\in \mathcal{F}} |F(f)-F_{NN}(f)| \leq C_6 (\log M)^{-\alpha\lambda} (\log \log M)^{(d-2)\alpha\lambda},
	\end{equation}
	where $C_4, C_5,C_6$ are positive constants.
\end{theorem}

It is noteworthy that the dominant term $(\log M)^{-\alpha\lambda}$ in the approximation rates is independent of $d$, and it only manifests in the negligible $\log\log M$ term. As a result, the overall rate experiences only a slight degradation as $d$ increases. This finding illustrates that our functional network can exploit the regularity of the input functions when they possess mixed smooth properties, particularly by choosing the embedding kernel in a data-dependent manner. It would be intriguing to explore additional scenarios where our functional network can exploit other specific characteristics of the input functions.

\section{Generalization Analysis}
\label{section4}

In this section, we conduct the theoretical analysis of the generalization error for the ERM algorithm applied to learn nonlinear functionals, utilizing our functional net as outlined in Definition \ref{deffunctionalNN}. {Moreover, the notation $\widetilde{O}$ is used to indicate that we conceal the additional logarithmic factors within the conventional $O$ notation.}

\subsection{{Problem settings and notations}}
We begin by defining the functional regression problem following the classical learning theory framework \citep{cucker2002mathematical,Cucker2007}. We assume that the first-stage data $D=\{f_i,y_i\}_{i=1}^m$ are i.i.d. samples drawn from the true unknown Borel probability distribution $\rho$ on $\mathcal{Z}=\mathcal{F}\times \mathcal{Y}$. Here, $\mathcal{F}$ represents the input function space, which is a compact subset of $L_\infty(\Omega)\cap L_2(\Omega)$ with $\Omega=[0,1]^d$, and satisfies the condition that $\|f\|_{L_2(\Omega)} \leq 1$ for any $f\in \mathcal{F}$, while $\mathcal{Y}=[-L,L]$ denotes the output space bounded by some constant $L>0$. However, in the case of functional data, the input functions cannot be observed directly; instead, we only have access to observations at discrete points. Thus, what we actually possess in practice are the second-stage data $\widehat{D}=\{\{t_{i,j},f_i(t_{i,j})\}_{j=1}^{n_i},y_i\}_{i=1}^m$, where $\{n_i\}_{i=1}^m$ indicates the sample size for the second stage.

Following the classical learning theory framework, our objective is to learn the regression functional
\begin{equation} \label{regressionfunctional}
	F_\rho(f)=\int_{\mathcal{Y}} y d\rho(y \vert f),
\end{equation}
where $\rho(y \vert f)$ is the conditional distribution at $f$ induced by $\rho$. This regression functional is the one that minimizes the generalization error using least squares loss
\begin{equation}
	\mathcal{E}(F)= \int_{\mathcal{Z}} \left(F(f)-y\right)^2 d \rho.
\end{equation}
We denote $\rho_{\mathcal{F}}$ as the marginal distribution of $\rho$ on $\mathcal{F}$, and $\left(L_{\rho_{\mathcal{F}}}^2, \Vert \cdot \Vert_\rho \right)$ as the space of square integrable functionals w.r.t.\ $\rho_{\mathcal{F}}$.

The hypothesis space $\mathcal{H}_{d_1,M}$ we use for the ERM algorithm is defined as
\begin{equation} \label{defhypospace}
	\begin{aligned}
		\mathcal{H}_{d_1,M}=  \{ & H_{NN} \circ T: H_{NN} \hbox{ is a structured deep ReLU neural network with input } \\
		& \hbox{dimension $d_1$, depth $J= d_1^2+d_1+1$, $M$ non-zero parameters, and} \\
		& \hbox{whose output has the following formulation \eqref{HNNoutput}} \}.
	\end{aligned}
\end{equation}
It takes the embedded function $L_K f$ as input, and $T$ serves as a projection step with the bases chosen as $\{\phi_i\}_{i=1}^{d_1}$, which consist of the eigenfunctions of the integral operator induced by the projection kernel $K_0$:
\begin{equation*} 
	T(g)
	=  \left[\int_{\Omega} g(t) \phi_1(t) d\mu(t), \int_{\Omega} g(t) \phi_2(t) d\mu(t), \cdots, \int_{\Omega} g(t) \phi_{d_1}(t) d\mu(t) \right]^T.
\end{equation*}

The architecture of the deep ReLU neural network $H_{NN}$ is specifically designed to approximate H\"older continuous functions on $[-R,R]^{d_1}$, and it has the explicit formulation given by
\begin{equation} \label{HNNoutput}
	H_{NN}(y)=\sum\limits_{i=1}^{(N+1)^{d_1}}c_i\psi\left(\frac{N}{2R}(y-b_i)\right), \quad y\in [-R,R]^{d_1},
\end{equation}
where $c_i\in \RR, b_i\in\mathbb{R}^{d_1}$ satisfying $|c_i| \leq L$, $\|b_i\|_\infty \leq R$, and $\psi: \RR^d \to \RR$ is defined as
\begin{equation}
	\psi(y)=\sigma\Big(\min\big\{\min_{k\neq j}(1+y_k-y_j),\min_{k}(1+y_k),\min_{k}(1-y_k)\big\}\Big).
\end{equation} 
Moreover, by Lemma \ref{lipschitzappro}, $N$ has the following relationship with $M$:
\begin{equation} \label{MNrelation}
	\bar C_1 d_1^4 (N+1)^{d_1} \leq M \leq \bar C_2 d_1^4 (N+1)^{d_1},
\end{equation}
where $\bar C_1$ and $\bar C_2$ are positive constants. Notice that
\begin{equation*} 
	\begin{aligned}
		\|T(L_K f)\|_\infty &\leq \|L_K f\|_{L_2{(\mu)}} \leq \sqrt{\mu(\Omega)} \|L_K f\|_{L_\infty(\Omega)} \\
		&= \sqrt{\mu(\Omega)} \sup_{x\in \Omega} \int_{\Omega} K(x,t) f(t) dt  \leq \sqrt{\mu(\Omega)} C_K \|f\|_{L_2(\Omega)} \leq \sqrt{\mu(\Omega)} C_K.
	\end{aligned}
\end{equation*}
Let $C_\mu=\mu(\Omega)$, then we can set $R=\sqrt{C_\mu} C_K$.
Denote $\mathcal{H}_{NN}$ as the function space of $H_{NN}$. The following proposition demonstrates that the difference in the outputs of $H_{NN}$ for different inputs can be bounded by the difference in those inputs.

\begin{proposition}\label{raodong}
	For any $h \in \mathcal H_{NN}$ and $x,y\in\mathbb{R}^{d_1}$, we have
	\begin{equation}
		|h(x)-h(y)|\le  \frac{L}{\sqrt{C_\mu}C_K}N(N+1)^{d_1} (d_1^2+d_1)||x-y||_1.
	\end{equation}
\end{proposition}

Furthermore, the empirical generalization error using the first-stage data $D=\{f_i,y_i\}_{i=1}^m$ is defined as
\begin{equation}
	\mathcal{E}_{D}(F)= \frac{1}{m} \sum_{i=1}^m \left(F(f_i)-y_i \right)^2.
\end{equation}
If we consider employing the hypothesis space defined in \eqref{defhypospace} along with the kernel embedding step within the first-stage ERM algorithm, the first-stage empirical generalization error can be formulated as
\begin{equation}
	\mathcal{E}_{D}(H\circ L_K)= \frac{1}{m} \sum_{i=1}^m \left(H \circ L_K f_i-y_i \right)^2.
\end{equation}
Moreover, if we denote
\begin{equation} 
	H_{D}= \mathop{\arg} \mathop{\min}_{H\in \mathcal{H}_{d_1,M}} \mathcal{E}_{D}(H\circ L_K),
\end{equation}
then the first-stage empirical target functional has the form 
\begin{equation}
	F_{D}=H_{D}\circ L_K.
\end{equation}

However, as previously noted, in practice, we can only acquire the second-stage data $\widehat{D}=\{\{t_{i,j},f(t_{i,j})\}_{j=1}^{n_i},y_i\}_{i=1}^m$. Therefore, rather than using the kernel embedding $L_K f_i$, we must rely on a quadrature scheme as outlined in \eqref{quadraturescheme}:
\begin{equation*}
	\widehat{L_K} f_i= \sum_{j=1}^{n_i} \theta_{i,j} K(\cdot,t_{i,j}).
\end{equation*}
Furthermore, if we consider utilizing the hypothesis space defined in \eqref{defhypospace} within the second-stage ERM algorithm, the second-stage empirical generalization error can be expressed as
\begin{equation}
	\mathcal{E}_{\widehat{D}}(H\circ \widehat{L_K})= \frac{1}{m} \sum_{i=1}^m \left(H\circ \widehat{L_K}f_i - y_i \right)^2.
\end{equation}
Similarly, if we denote
\begin{equation} 
	H_{\widehat{D}}= \mathop{\arg} \mathop{\min}_{H\in \mathcal{H}_{d_1,M}} \mathcal{E}_{\widehat{D}}(H\circ \widehat{L_K}),
\end{equation}
then the second-stage empirical target functional has the form 
\begin{equation} \label{defemptarg}
	F_{\widehat{D}}=H_{\widehat{D}}\circ L_K.
\end{equation}

Finally, we define the projection operator $\pi_{L}$ on the functional space $F: \mathcal{F} \to \mathbb{R}$ as
$$
\pi_{L}(F)(f)= \begin{cases}L, & \text { if } F(f)>L, \\ -L, & \text { if } F(f)<-L, \\ F(f), & \text { if }-L \leq F(f) \leq L .\end{cases}
$$
Since the regression functional $F_{\rho}$ is bounded by $L$,  we will use the truncated empirical target functional $$\pi_{L} F_{\widehat{D}}$$ as the final estimator.

\subsection{{A two-stage oracle inequality}}

Given that the ERM algorithm operates on the second-stage data for the FDA, in contrast to traditional regression problems that only consider first-stage data, we must develop a new oracle inequality for the generalization analysis of the ERM algorithm applied to our functional network. The primary approach involves employing a two-stage error decomposition method, where the first-stage empirical error serves as an intermediary term in this decomposition of errors.

In the following, for the convenience, we denote $\mathcal{H}=\mathcal{H}_{d_1,M}$, and $u_i=L_Kf_i$, $\hat u_i=\widehat{L_K}f_i$, for $i=1,2,\dots,m$.
\begin{proposition}\label{errordec}
	Let $F_{\widehat D}$ be the empirical target functional defined in \eqref{defemptarg}, and $\pi_{L} F_{\widehat{D}}$ be the truncated empirical target functional, then
	\begin{equation}
		\begin{aligned}
			& \mathcal{E}\left(\pi_{L}F_{\widehat D} \right)- \mathcal{E}\left(F_\rho \right)\leq I_1+ I_2 + I_3.
		\end{aligned}
	\end{equation}
	where
	\begin{equation}
		\begin{aligned}
			& I_1 = \left\{\mathcal{E}\left(\pi_{L}F_{\widehat D} \right)- \mathcal{E}\left(F_\rho \right) \right\}- 2\left\{\mathcal{E}_{D}\left(\pi_{L}F_{\widehat D} \right)-\mathcal{E}_{D}\left(F_\rho \right) \right\}, \\
			& I_2 = 2\left\{\mathcal{E}_{D}\left(\pi_{L} F_{D} \right)-\mathcal{E}_D\left(F_\rho \right)\right\}, \\
			& I_3= 16L\sup_{H\in \mathcal{H}}\frac{1}{m}\sum_{i=1}^{m}\left|H(u_i)-H(\hat{u}_i)\right|.
		\end{aligned}
	\end{equation}
\end{proposition}

{Our next objective is to establish a two-stage oracle inequality for analyzing the generalization ability of the ERM algorithm applied to our functional network. This will be influenced by the capacity of the hypothesis space we employ, as well as the quadrature scheme $\widehat{L_K} f$ in \eqref{quadraturescheme}, which is utilized for approximating $L_K f$ as described in \eqref{kernelquadrature}}. 

{
	In this theoretical analysis, we consider the utilization of a theoretically optimal kernel quadrature scheme described in \cite[pp. 21--22]{bach2017equivalence}. To achieve an approximation accuracy $\epsilon$, the quadrature formula is given by:
	\begin{equation} \label{optimalquadrature}
		\widehat{L_K} f_i= \sum_{j=1}^{n_i} \theta_{i,j} K(\cdot,t_{i,j}),
	\end{equation}
	where the observation points $\{t_{i,j}\}_{j=1}^{n_i}$ are i.i.d. sampled from an optimal distribution on $\Omega$ with density $\tau$ w.r.t.\ $d \mathcal{U}$ satisfying
	\begin{equation} \label{optimaldistribution}
		\tau(x) \propto \sum_{k\geq 1} \frac{\lambda_k}{\lambda_k+ \epsilon} \phi_k(x)^2,
	\end{equation}
	and the weights $\{\theta_{i,j}\}_{j=1}^{n_i}$ are computed by minimizing 
	$$	\sum_{j=1}^{n_i} \sum_{k=1}^{n_i} \theta_{i,j} \theta_{i,k} K(t_{i,j},t_{i,k}) - 2 \sum_{j=1}^{n_i} \theta_{i,j} \int_\Omega K(t,t_{i,j}) f_i(t) d \mathcal{U}(t), $$
	subject to $\sum_{j=1}^{n_i} \theta_{i,j}^2 \leq 4/n$ (\cite[pp. 21--22]{bach2017equivalence}, \cite[pp. 3--4]{kanagawa2016convergence}). 
}

We are now prepared to present the two-stage oracle inequality for the ERM algorithm when the hypothesis space is defined by our functional network. In this context, the embedding kernel $K$ and the projection kernel $K_0$ are selected as specified in Section \ref{section2}.
\begin{theorem}\label{oracleinequality}
	Let $m,n \in\mathbb{N}$, $\Omega=[0,1]^d$. Assume that the second-stage sample sizes are equal, specifically $n_1=n_2= \cdots =n_m=n$, the bound of the output space $L\ge 1$, the input function $f \in L_\infty(\Omega)\cap L_2(\Omega)$ and satisfies the norm constraint $\|f\|_{L_2(\Omega)} \leq 1$. Let $F_{\widehat D}$ denote the empirical target functional defined in \eqref{defemptarg}, $F_\rho$ represent the regression functional defined in \eqref{regressionfunctional}, then we have the following inequality
	\begin{equation}
		\mathop{E}||\pi_LF_{\widehat D}-F_{\rho}||^2_{\rho}\le \frac{c'_1 JM\log M\log m}{m} + \frac{c'_3}{C_K} M^{1+\frac{1}{d_1}} \sqrt{\lambda_q} + 2\inf_{F\in \{H\circ L_K: H\in \mathcal{H}\}} \|F-F_{\rho}\|_{\rho}^2,
	\end{equation}
	where the expectation is taken with respect to the first-stage data $D$ and the second-stage data $\widehat{D}$, $\{\lambda_i\}$ are the eigenvalues of the integral operator $L^\mu_{K_0}$ associated with the projection kernel $K_0$. Moreover, $q= \frac{c'_5 n}{\log n}$,  $J=d_1^2+d_1+1$, $C_K= \sup_{x,t} |K(x,t)|$ depending on the embedding kernel $K$, and $c'_1,c'_3,c'_5 $ are positive constants.
\end{theorem}


\subsection{{Generalization error bounds}}

Building on the two-stage oracle inequality outlined in Theorem \ref{oracleinequality}, we are now prepared to establish the generalization error bounds for the ERM algorithm applied to our functional network. We first examine the generalization analysis scenario where the input function space is Sobolev space, by leveraging the approximation results presented in Section \ref{section21}. For the dimension reduction step in this context, the embedding kernel $K$ is selected as defined in \eqref{embedkernel1}, the projection kernel is selected as $K_0=k_\gamma$, and the hyperparameters $d_1, \gamma$ are selected according to \eqref{hyper1} within Theorem \ref{approxrate12}.
\begin{theorem} \label{generalizationerror1}
	Let $\alpha, m,n \in\mathbb{N}$, $\Omega=[0,1]^d$. Assume that the input function space $\mathcal{F}$ is a compact subset of $H_0^\alpha(\Omega)$, with the condition that $\|f\|_{W_{2}^\alpha(\Omega) }\leq 1$ for any function $f\in \mathcal{F}$, and the modulus of continuity of the regression functional $F_\rho: \mathcal{F} \to \RR$ satisfies the condition (\ref{functionalsmoothness}) with $\lambda\in (0,1]$. Additionally, suppose that the second-stage sample sizes are equal, specifically $n_1=n_2= \cdots =n_m=n$.  By choosing the number of nonzero parameters $M$ in the functional network and the second-stage sample size $n$ as
	\begin{equation}
		M = \left \lfloor \frac{m}{(\log m)^{\frac{2\alpha\lambda}{d}+4}} \right \rfloor, \qquad  	n \geq \hat C_2 (\log m)^d \log \log m,
	\end{equation}
	we obtain the following generalization error bound
	\begin{equation}
		\mathop{E}||\pi_LF_{\widehat D}-F_{\rho}||^2_{\rho}\leq \tilde{C}_2 \left(\log m\right)^{-\frac{2\alpha\lambda}{d}}  \left(\log\log m\right)^{2\left(\frac{1}{d}+1\right) \alpha\lambda} ,
	\end{equation}
	where $\tilde{C}_2, \hat{C}_2$ are positive constants.
\end{theorem}

Since $\log\log m$ is negligible compared to $\log m$, the generalization error bounds we obtain for the Sobolev input function spaces converge at rates of $\widetilde{O}((\log m)^{-\frac{2\alpha\lambda}{d}})$. Similarly, we can derive the learning rates for input function spaces that are Gaussian RKHSs and mixed smooth Sobolev spaces by applying the approximation results from Section \ref{section22} and Section \ref{section23}, respectively. The proofs of the following two theorems are omitted, as they follow the same approach as in the proof of Theorem \ref{generalizationerror1}. In Theorem \ref{generalizationerror2}, we select the embedding kernel $K$ as defined in \eqref{embeddingkernelg} and the projection kernel $K_0$ as the Gaussian kernel $ k_\gamma$, leading to an eigenvalue decay of $\lambda_q = O(e^{-2c_2 q^{\frac{1}{d}}})$, and the hyperparameter $d_1$ is selected according to \eqref{hyper2} within Theorem \ref{approxrate2}. In Theorem \ref{generalizationerror3}, we select the embedding kernel $K$ as defined in \eqref{embeddingkernelm} and the projection kernel $K_0$ as the Mat\'ern  kernel $ K_\alpha$, resulting in an eigenvalue decay of $\lambda_q = O(q^{-2\alpha})$, and the hyperparameter $d_1$ is selected according to \eqref{hyper3} within Theorem \ref{approxrate3}.

\begin{theorem} \label{generalizationerror2}
	Let $\gamma>0, m,n \in\mathbb{N}$, $\Omega=[0,1]^d$. Assume that the input function space $\mathcal{F}$ is a compact subset of  the unit ball of the Gaussian RKHS $H_\gamma(\Omega)$, and the modulus of continuity of the regression functional $F_\rho: \mathcal{F} \to \RR$ satisfies the condition (\ref{functionalsmoothness}) with $\lambda\in (0,1]$. Additionally, suppose that the second-stage sample sizes are equal, specifically $n_1=n_2= \cdots =n_m=n$.  By choosing the number of nonzero parameters $M$ and second-stage sample size $n$ as
	\begin{equation}
		M = \left\lfloor \frac{m e^{-2c_7 \lambda \left(\log m \right)^{\frac{1}{d+1}}}}{(\log m)^{4}} \right\rfloor, \qquad 
		n \geq \hat C_3 (\log m)^d \log \log m,
	\end{equation}
	we obtain the following generalization error bound
	\begin{equation}
		\mathop{E}||\pi_LF_{\widehat D}-F_{\rho}||^2_{\rho}\leq \tilde{C}_3 e^{-2c_7 \lambda \left(\log m \right)^{\frac{1}{d+1}}} \left(\log m \right)^{\frac{2d}{d+1}} ,
	\end{equation}
	where $\tilde{C}_3$, $\hat C_3$,  $c_7$ are positive constants with $c_7> \left(\frac{d}{3e}\right)^{\frac{d}{d+1}}$.
\end{theorem}

Since $(\log m)^{\frac{2d}{d+1}}$ is negligible compared to $e^{-2c_7 \lambda(\log m)^{\frac{1}{d+1}}}$, the generalization error bounds obtained for the Gaussian RKHS input function space exhibit convergence rates of $\widetilde O(e^{-2c_7 \lambda(\log m)^{\frac{1}{d+1}}})$, with $c_7 > \left(\frac{d}{3e}\right)^{\frac{d}{d+1}}$. This rate is superior to $(\log m)^{-a}$ for any $a > 0$, yet inferior to $m^{-a}$ for any $a > 0$ in the asymptotic context as $m \to \infty$. This suggests that we are still unable to achieve polynomial learning rates for infinitely differentiable functions.


\begin{theorem} \label{generalizationerror3}
	Let $\alpha, m,n \in\mathbb{N}$, $\Omega=[0,1]^d$. Assume that the input function space $\mathcal{F}$ is a compact subset of  the unit ball of the mixed smooth Sobolev space $H_{mix}^\alpha(\Omega)$, and the modulus of continuity of the regression functional $F_\rho: \mathcal{F} \to \RR$ satisfies the condition (\ref{functionalsmoothness}) with $\lambda\in (0,1]$. Additionally, suppose that the second-stage sample sizes are equal, specifically $n_1=n_2= \cdots =n_m=n$. By choosing the number of nonzero parameters $M$ and second-stage sample size $n$ as
	\begin{equation}
		M = \left\lfloor\frac{m}{(\log m)^{2\alpha\lambda+4}}\right\rfloor, \qquad 
		n\geq \hat C_4 m^{\frac{1}{\alpha}} (\log m)^{1+\frac{1}{C_4 \alpha}+2\lambda-2d\lambda-\frac{4}{\alpha}},
	\end{equation}
	we obtain the following generalization error bound
	\begin{equation}
		\mathop{E}||\pi_LF_{\widehat D}-F_{\rho}||^2_{\rho}\leq \tilde{C}_4 \left(\log m\right)^{-2\alpha\lambda}  \left(\log\log m\right)^{2(d-2) \alpha\lambda},
	\end{equation}
	where $\hat C_4$, $C_4$, $\tilde{C}_4$ are positive constants.
\end{theorem}

Since $\log\log m$ is negligible compared to $\log m$, the generalization error bound we obtain for the mixed smooth Sobolev input function spaces shows convergence rates of $\widetilde{O}((\log m)^{-2\alpha\lambda})$. This rate is independent of the dimension $d$ of the input function spaces, similar to that in the approximation rates.

\section{{Related Work and Summary}}

{
	Numerous studies have explored the applications of deep learning in the FDA. For instance, in the context of solving parametric partial differential equations (PDEs), works such as \cite{khoo2021solving} and \cite{lu2021learning} focus on using operator neural networks to learn the relationship between the parametric function space and the solution space. Regarding inverse scattering problems, \cite{khoo2019switchnet} and \cite{wei2019physics} employ deep learning techniques to learn an operator that maps the observed data function space to the parametric function space that characterizes the underlying PDE for high-frequency phenomena. In the field of signal processing, research works like \cite{andreotti2016open} and \cite{grais2017single} have investigated using deep learning to learn a nonlinear operator that maps a signal function to one or multiple other signal functions. Additionally, there have been various efforts in image processing tasks utilizing deep neural networks, including phase retrieval \citep{deng2020learning}, image super-resolution \citep{qiao2021evaluation}, image inpainting \citep{qin2021image}, and image denoising \citep{tian2020deep}. }

{
	Discretization invariant learning is an important area of research within the FDA. It focuses on learning within infinite-dimensional function spaces while effectively managing various discrete representations of functions either as inputs or outputs in a learning model. To develop mesh-independent models, a traditional approach involves pre-processing the input data and post-processing the output data to ensure that the processed data aligns with the requirements of deep neural networks. Common techniques for processing include interpolation, padding, resizing \citep{keys1981cubic}, and cropping \citep{ravuri2021skilful}. Recent studies have introduced additional methods for discretization invariant learning that utilize kernel integral operators. In this framework, the input function $f$ is transformed to the next layer via a linear integral transformation represented as $v(x)=\int_{\Omega} K(x,t;\theta)f(t) dt$, followed by a nonlinear activation. Here, $K(x,t;\theta)$ denotes the integral kernel parameterized by $\theta$. This integral transformation can be applied independently of the discretization of the input function $f$. The choice of the integral kernel can include convolutional kernels, leading to pure convolutional neural networks \citep{ronneberger2015u,guo2016convolutional,khoo2021solving}, as well as parameterized neural networks \citep{anandkumar2020neural,li2020multipole,li2020fourier,ong2022integral} that are utilized in tasks such as image classification, solving parametric PDEs, and addressing initial value problems. }

{
	We note that from another perspective, employing kernel integral transformation with a smooth kernel can be regarded as a method of kernel smoothing \citep{wand1994kernel}. This technique is commonly used in various domains, including image processing \citep{chung2013statistical} and functional linear regression based on FPCA \citep{yao2005functional, zhou2022functional}. For instance, in the context of image processing, selecting $K$ as a translation-invariant kernel effectively results in a convolution. In functional linear regression, an empirical integral transformation defined as $\widehat{L_K} f_i= \frac{1}{n_i} \sum_{j=1}^{n_i} f_i(t_{i,j}) K(\cdot,t_{i,j})$ serves as a pre-smoothing technique for the input function data that is observed discretely, leveraging a smoothing density kernel. }

{
	The theoretical exploration of artificial neural networks has been ongoing for over thirty years. Recent studies have quantitatively shown that deep ReLU neural networks possess approximation capabilities on par with traditional deep neural networks that utilize infinitely differentiable activation functions, such as sigmoid and tanh functions \citep{telgarsky2016benefits,Yarotsky2017,suzuki2018adaptivity,zhou2020acha,mao2022approximating}. Regarding approximation results in FDA, the foundational result is the universal approximation theorem for nonlinear operators established in \cite{chen1995universal}. Subsequent researches by \cite{mhaskar1997neural}, \cite{bhattacharya2021model}, \cite{kovachki2021universal}, and \cite{lanthaler2022error} have further quantitatively examined the expressivity of deep neural networks in approximating operators. More recently, \cite{mhaskar2023local} investigated the local approximation of operators through deep neural networks. Building on these approximation results, generalization analyses for the ERM algorithm over deep ReLU neural networks have been thoroughly developed in \citep{chui2019deep,SchmidtHieber2020,msz2021}. Recent advancements in FDA have taken a further step by analyzing the generalization error in operator learning using deep neural networks \citep{lanthaler2022error,de2023convergence}. In addition, \cite{liu2024deep} has focused on the nonparametric estimation of Lipschitz operators with deep neural networks, providing non-asymptotic bounds for the generalization error of the ERM algorithm based on a suitably selected class of networks. }

\begin{table}[t]
	\centering
	\begin{tabular}{|c|c|c|}  
		\hline
		{Function classes} & {Approximation error} & {Estimation error} \\  
		\hline
		{Sobolev spaces $H_0^\alpha(\Omega)$} & {$\widetilde O\left((\log M)^{-\frac{\alpha\lambda}{d}}\right)$} & {$\widetilde O\left((\log m)^{-\frac{2\alpha\lambda}{d}}\right)$} \\  
		\hline
		{Gaussian RKHSs $H_\gamma(\Omega)$} & {$\widetilde O\left(e^{-c_7 \lambda \left(\log M \right)^{\frac{1}{d+1}}}\right)$} & {$\widetilde O\left(e^{-2c_7 \lambda(\log m)^{\frac{1}{d+1}}}\right)$}  \\ 
		\hline
		\makecell[c]{{mixed smooth} \\ {Sobolev spaces $H_{mix}^\alpha(\Omega)$}} & {$\widetilde{O}\left((\log M)^{-\alpha\lambda}\right)$} & {$\widetilde O\left((\log m)^{-2\alpha\lambda}\right)$}  \\  
		\hline
	\end{tabular}
	\caption{{Summary of approximation and learning rates achieved by our functional net. $M$ denotes the number of non-zero parameters in the functional net, and $m$ represents the first-stage sample size.}}
	\label{summary}
\end{table}

{
	In this paper, we improve the theoretical comprehension of the FDA by examining the expressiveness and generalization capabilities of the functional deep neural network with kernel embedding that we propose. As detailed in Table \ref{summary}, we establish explicit rates for approximating nonlinear smooth functionals across various input function spaces, such as Sobolev spaces, Gaussian RKHSs, and mixed-smooth Sobolev spaces. Based on these approximation results, we further derive explicit learning rates for the ERM algorithm applied to our functional network, when the second-stage sample size $n$ is sufficiently large based on the employed quadrature scheme. This is achieved through a novel two-stage oracle inequality that takes into account both the first-stage sample size $m$ and the second-stage sample size. }

\section{Numerical Simulations}
\label{section5}

{In this section, we conduct standard numerical simulations to evaluate the performance and advantages of our proposed functional network with kernel embedding (KEFNN), which is trained as outlined in Algorithm \ref{algorithm} in Section \ref{section1}. We start by examining the numerical effectiveness of our model through simulated examples, comparing it with functional networks that utilize different dimension reduction methods, as well as with a standard deep neural network. Following this, we carry out numerical simulations on synthetic data to gain further insights into our model's behavior. This includes assessing the impact of the first-stage sample size, second-stage sample size, network depth, and noise levels on performance, validating the discretization invariant property, and exploring the effects of various quadrature schemes on performance. Finally, we evaluate the effectiveness of KEFNN on several small real functional datasets.}

\subsection{Comparison with baseline approaches} 

To evaluate the performance of our model against baseline approaches, we adopt the same data-generating process used in \cite{yao2021deep}, which is commonly employed in functional data simulations. The input random function is generated through the process $f(t)=\sum_{k=1}^{50} c_k \phi_k(t)$ with $t\in [0,1]$, where $\phi_1(t)=1$,  $\phi_k(t)= \sqrt{2} cos((k-1)\pi t)$ for $k\in \{2,3,\dots,50\}$. The coefficients are defined as $c_k= z_k r_k$, with $r_k$ being i.i.d. uniform random variables on $[-\sqrt{3},\sqrt{3}]$. The first-stage data consists of 4000 samples, while the second-stage data of $f_i(t)$ are observed at discrete points $\{t_{i,j}\}_{j=1}^{51}$ that are equally spaced on $[0,1]$. The training, validation, and test split follows a ratio of $64:16:20$. 

We assess the performance of our method in comparison to some baseline approaches that utilize different dimension reduction techniques, including the direct use of discrete observation points (``Raw data") \citep{rossi2002functional}, projection by basis approaches employing ``B-spline" basis \citep{rossi2005representation}, ``FPCA" basis \citep{rossi2005representation}, and the neural network basis utilized in ``AdaFNN" \citep{yao2021deep}. It is important to note that the number of bases in ``B-spline" and ``FPCA" should not be too small (as this would fail to capture the main information in the input function) or too large (which could lead to overfitting due to noisy observations). In contrast, our method leverages kernel embedding as a pre-smoothing step, allowing for a larger number of eigenfunctions during the projection step, thus retaining as much useful information as possible without being affected by noise. The number of bases in AdaFNN is limited to a maximum of four, as the target functionals depend on at most two bases in the cases considered in the simulation. We provide the mean squared error (MSE) on the test set for different methods in Table \ref{table1}, including the best test-set MSE reported in \cite{yao2021deep} for the other approaches.

For \textit{case 1}, we set $z_1= z_3=5$, $z_5=z_{10}=3$, and $z_k=1$ for other $k$. The response $Y=(\langle f,\phi_5 \rangle)^2= c_5^2$ has a nonlinear relationship with the input function and only depends on one Fourier basis. We don't consider noise in this case. For \textit{case 2}, the other settings are the same as case 1, except that the observations of $f_i(t)$ at each point include a Gaussian noise $N(0,\sigma_1^2)$ with $\sigma_1^2= 11.4$, and the observations of the response $Y$ also have a Gaussian noise $N(0,\sigma_2^2)$ with $\sigma_2^2= 0.3$. For \textit{case 3}, $z_k=1$ for all $k$, and $Y= \langle f,\beta_2 \rangle+ (\langle f,\beta_1 \rangle)^2$, with $\beta_1(t)= (4-16t)\cdot 1\{0\leq t\leq 1/4\}$ and $\beta_2(t)= (4-16|1/2-t|) \cdot 1\{1/4 \leq t \leq 3/4\}$. The noises are also included with $\sigma_1^2= 5$, $\sigma_2^2= 0.1$. For \textit{case 4}, the settings are the same as case 3, except that the noises in the observation of $Y$ are enlarged, i.e., the variance is doubled with $\sigma_2^2= 0.2$.

The deep neural network architectures after the dimension reduction approach are consistent across all methods, each comprising three hidden layers with 128 neurons in each layer, implemented using PyTorch. The functional input and response data are standardized entry-wise for all learning tasks. All functional networks are trained for 500 epochs, with the best model selected based on validation loss. The Adam optimizer is employed for optimization, with a learning rate set at $3\times 10^{-4}$. In our model, we use a Gaussian kernel with bandwidth $\gamma$ as the embedding kernel and utilize the first $d_1$ eigenfunctions of its integral operator, explicitly formulated using Hermite polynomials \cite[pp. 28]{Fasshauer2011}, as the bases for the projection step. The measure $\mu$ in the integral operator is characterized by a Gaussian distribution, whose density function is $u(x)= \frac{\beta}{\sqrt{\pi}} \exp^{-\beta^2 x^2}$, where $\beta$ acts as a global scaling parameter. {The weights $\{\theta_{i,j}\}_{j=1}^n$  are determined using the composite trapezoidal rule.}  Numerical integration is employed to compute the coefficient vector for the projection step. The hyperparameters $\gamma,\beta,d_1$ are selected through cross-validation over a 3D grid of parameters. For the first three cases, the optimal hyperparameters are $\gamma= 0.033$, $\beta=0.008$, $d_1= 150$. In case 4, the other hyperparameters remain the same, with the exception that $\gamma= 0.02$.


\begin{table}
	\centering
	\caption{Comparison of the test-set MSE for functional deep neural networks utilizing different dimension reduction methods.} \label{table1}
	\begin{tabular}{cccccc}
		\toprule  
		&Case 1  & 	Case 2  & 	Case 3   &  Case 4   \\
		\midrule  
		Raw data+NN    & 0.038   &  0.275   &  0.334  &  0.339  	\\
		B-spline+NN    & 0.019   &  0.206   &  0.251  &   0.257 	\\
		FPCA+NN    & 0.023   &  0.134   &  0.667  &  0.693  	\\
		AdaFNN    & 0.003   &  0.127   &  0.193  &   0.207 	\\
		KEFNN    & \textbf{0.001}   &  \textbf{0.100}   &  \textbf{0.142}  &  \textbf{0.191}	\\
		\bottomrule 
	\end{tabular}
\end{table}

From Table \ref{table1}, it is evident that KEFNN consistently surpasses other dimension reduction techniques, primarily due to the careful tuning of hyperparameters utilizing the information from the entire dataset. In these simulations, Gaussian kernels are exclusively employed as the embedding kernel; however, exploring alternative embedding kernels could potentially enhance performance further. Additionally, as shown by the numerical computation \eqref{numericalke} of the coefficient vector in our dimension reduction method, it relies solely on hyperparameters without introducing any training parameters that might escalate computational demands—such as those incurred by eigenfunction estimation in FPCA or the optimization of free parameters in the neural network basis of AdaFNN. Moreover, dimension reduction methods based on B-spline and FPCA depend purely on input function data while ignoring response data, which could result in the selection of bases misaligned with the true target functional. Conversely, in our dimension reduction strategy, hyperparameters are determined using both input and response data. These aspects collectively account for the superior performance of KEFNN compared to the baseline approaches.

\subsection{Additional insights on KEFNN}

In this subsection, we conduct numerical simulations to further explore the behavior of our model. Specifically, we analyze the impact of varying the first-stage sample size $m$ and the second-stage sample size $n$, the depth of KEFNN, the level of observational noise in input functions and responses, as well as the effects of discretization and quadrature schemes on generalization performance. The target functional remains consistent with case 3 from the preceding subsection, expressed as $Y= \langle f,\beta_1 \rangle+ (\langle f,\beta_2 \rangle)^2$. {Unless otherwise indicated, the discrete points are uniformly spaced across the interval $[0,1]$.} Furthermore, the hyperparameter settings in this subsection are maintained as $\gamma= 0.02$, $\beta=0.008$, and $d_1= 150$.

We begin by examining the effect of the second-stage sample size $n$ on generalization performance, with the first-stage sample size fixed at $m= 4000$ and the noise variances set to $\sigma_1^2=3$ and $\sigma_2^2=0.05$. As illustrated in Figure \ref{figures}(a), a phase transition phenomenon is observed in our model around $n=4000$, indicating that further increasing $n$ beyond this point does not lead to improvements in generalization performance. {This requirement for the second-stage sample size is approximately equivalent to the first-stage sample size when trapezoidal rules are applied for the quadrature problem.}

\begin{figure}
	\centering
	\subfigure[Test-set MSE w.r.t. second-stage sample size]{
		\centering\includegraphics[width=0.42\textwidth]{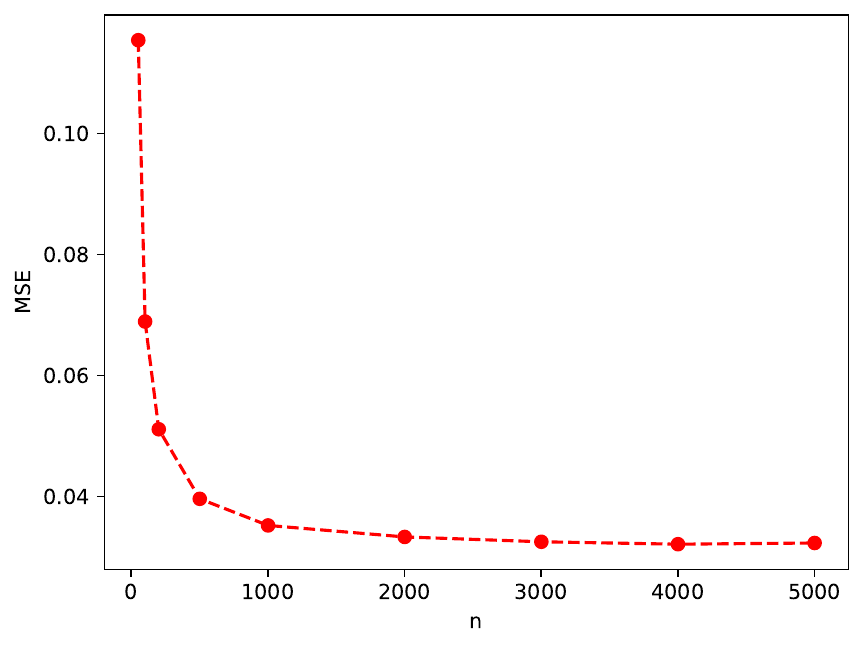}
	}
	\quad
	\subfigure[Test-set MSE w.r.t. depth of KEFNN]{
		\includegraphics[width=0.43\textwidth]{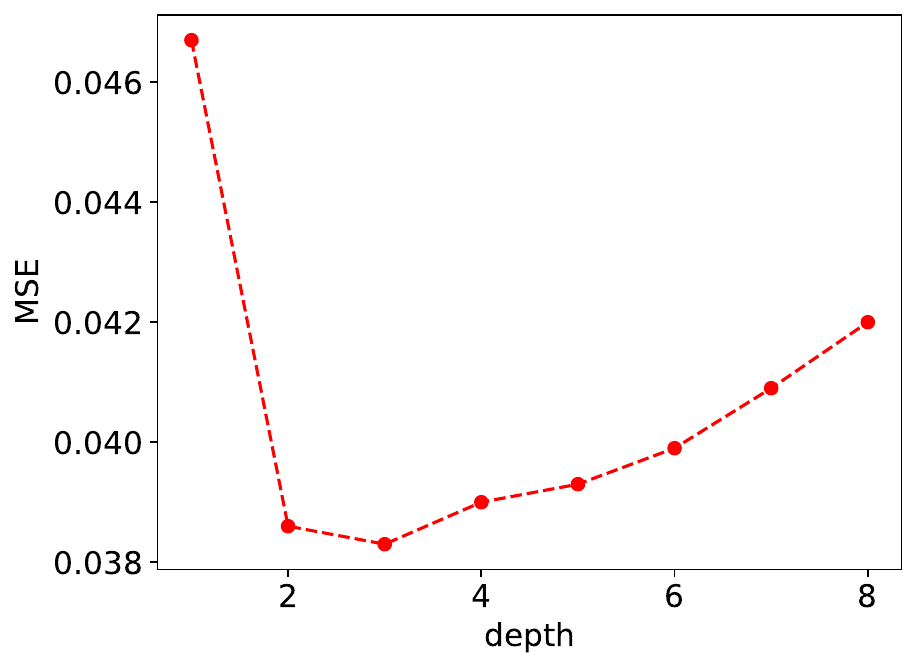}
	}
	\quad
	\subfigure[Test-set MSE w.r.t. first-stage sample size]{
		\includegraphics[width=0.447\textwidth]{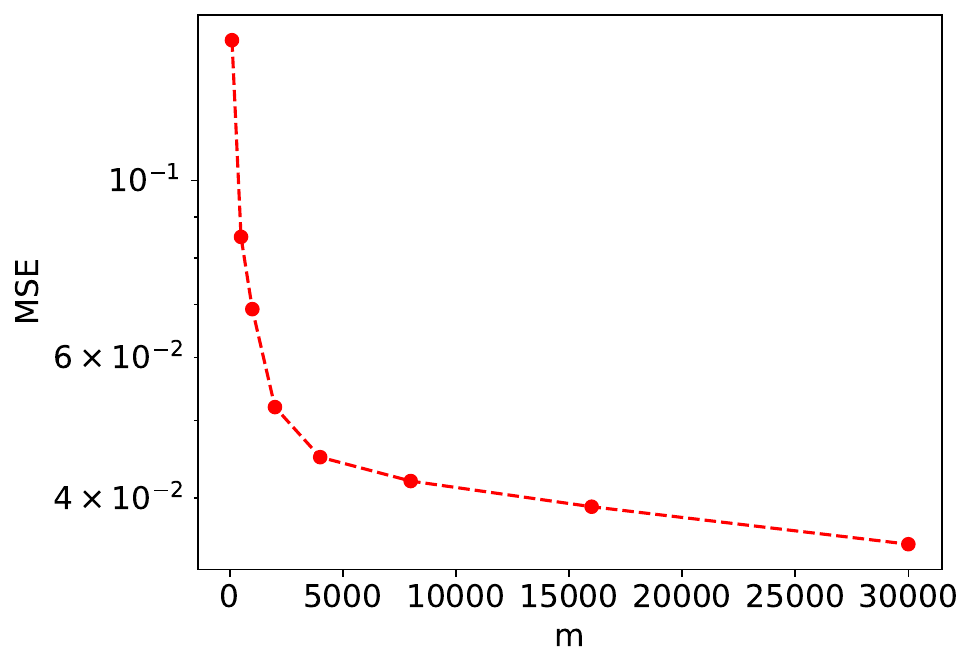}
	}
	\quad
	\subfigure[Inverse of test-set MSE w.r.t. log of first-stage sample size]{
		\includegraphics[width=0.4\textwidth]{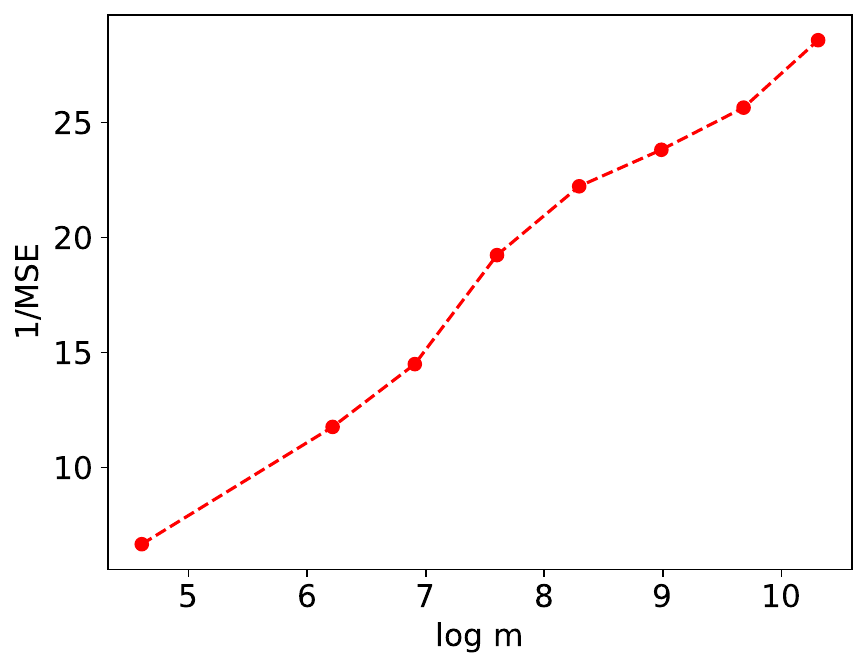}
	}
	\quad
	\subfigure[Test-set MSE w.r.t. the variance of Gaussian noises in observations of input functions]{
		\includegraphics[width=0.42\textwidth]{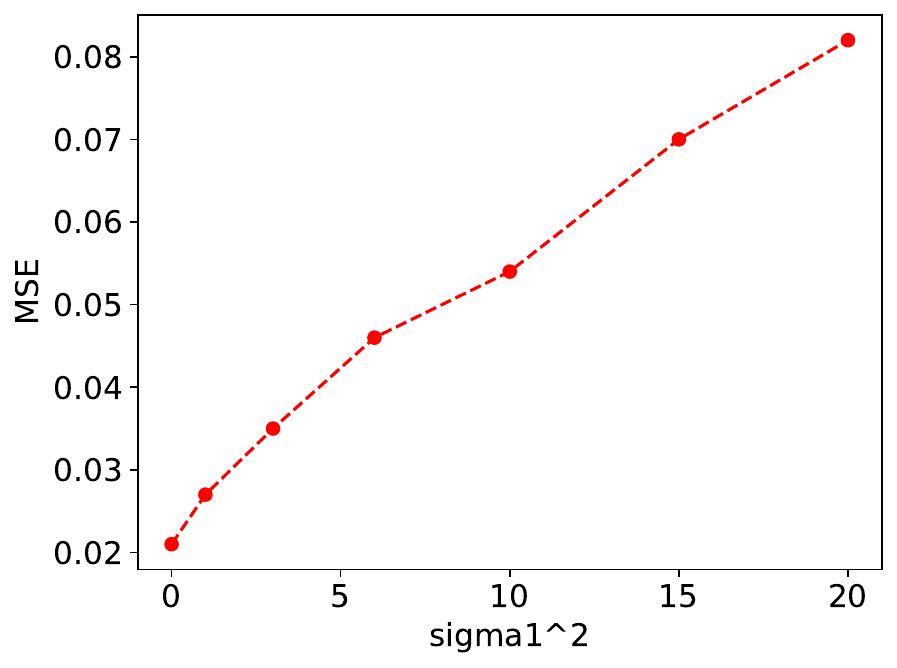}
	}
	\quad
	\subfigure[Test-set MSE w.r.t. the variance of Gaussian noises in observations of responses]{
		\includegraphics[width=0.42\textwidth]{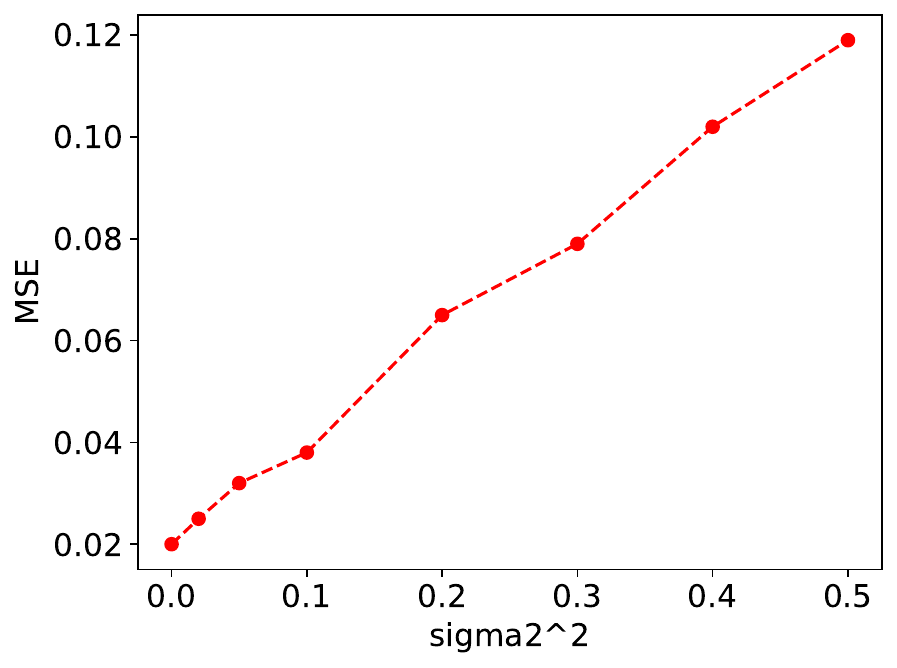}
	}
	\caption{The influence of second-stage sample size $n$, depth of KEFNN, first-stage sample size $m$, variances of Gaussian noises in observations of input functions and responses on the test-set MSE.}
	\label{figures}
\end{figure}

Next, we investigate how the depth of KEFNN influences generalization performance within the under-parameterized regime. For this study, we set the first-stage sample size to $m= 5000$, the second-stage sample size to $n=500$, and the noise variances to $\sigma_1^2=3$ and $\sigma_2^2=0.05$. The width of the deep neural network in KEFNN is fixed at 16. Figure \ref{figures}(b) reveals that as the depth increases, the generalization error initially decreases but then begins to rise. This behavior aligns with our theoretical analysis in Section \ref{section4}, which suggests that the generalization error is minimized when the number of nonzero parameters in KEFNN follows a specific rate relative to the first-stage sample size. This finding provides partial guidance on how to conduct model selection for KEFNN.

In Figure \ref{figures}(c), we illustrate the effect of the first-stage sample size $m$ on generalization performance, where the second-stage sample size is fixed at $n=500$, and the noise variances are set to $\sigma_1^2=3$ and $\sigma_2^2=0.05$. To further highlight the decay rate of the generalization error with respect to increasing $m$, we plot the inverse of the test-set MSE against the $\log$ of the first-stage sample size $m$ in Figure \ref{figures}(d). From this, we observe that the generalization error decreases at a rate of approximately $O((\log m)^{-1})$, which corresponds to a polynomial rate in $\log m$. This observation aligns with the learning rates established in Section \ref{section4}. Such polynomial rates in $\log m$ arise from the curse of dimensionality while approximating nonlinear smooth functionals. Therefore, it is crucial to investigate the conditions under which this curse of dimensionality can be mitigated for functional data.

Next, we analyze the impact of noise in the observations of input functions and responses on generalization performance. The variation in test-set MSE loss with respect to the variances of Gaussian noise in the observations of input functions and responses is depicted in Figure \ref{figures}(e) and Figure \ref{figures}(f), respectively. For this simulation, we set the first-stage sample size to $m= 4000$ and the second-stage sample size to $n= 500$. In Figure \ref{figures}(e), the variances of noise in responses are fixed at $\sigma_2^2=0$, while in Figure \ref{figures}(f), the variances of noise in input functions are fixed at $\sigma_1^2=0$. Interestingly, in both scenarios, we observe that the generalization error increases almost linearly with the variances of Gaussian noise in the observations. This phenomenon prompts further theoretical exploration to offer a rigorous explanation. Moreover, it underscores that employing kernel embedding as a pre-smoothing technique imparts a certain level of robustness to KEFNN against observational noise in both input functions and responses.

\begin{table} [t]
	\centering
	\caption{{The test-set MSE for different discretization schemes with varying second-stage sample sizes ($n$) across three cases of first-stage sample sizes ($m$).}} \label{table2}
	{
		\begin{tabular}{cccccc}
			\toprule  
			&	n=3965 & n=3987  & n=4002  &  n=4034   &  n=4042   \\
			\midrule  
			m=1000    & 0.0669   &  0.0640   &  0.0665  &  0.0650  &  0.0659	\\
			m=2000   & 0.0519   &  0.0496   &  0.0499  &   0.0513  &  0.0496	\\
			m=3000    & 0.0470   &  0.0476   &  0.0475  &  0.0462   &  0.0467 	\\
			\bottomrule 
		\end{tabular}
	}
\end{table}

{We further investigate the impact of various discretizations on the generalization performance of KEFNN. The noise variances are fixed at $\sigma_1^2=3$ and $\sigma_2^2=0.05$, and we assess the performance under different discretizations with varying second-stage sample sizes: $n=3965$, $n=3987$, $n=4002$, $n=4034$, and $n=4042$, across three cases of first-stage sample sizes: $m=1000$, $m=2000$, and $m=3000$. The discretization points are i.i.d.\ sampled from a uniform distribution on $[0,1]$. The averaged test-set MSE over five trials is presented in Table \ref{table2}. The results suggest that KEFNN demonstrates a property of discretization invariance, since functional networks trained with different levels of discretization yield similar generalization performance.
}

\begin{figure}[t]
	\centering
	\includegraphics[width=0.9\textwidth]{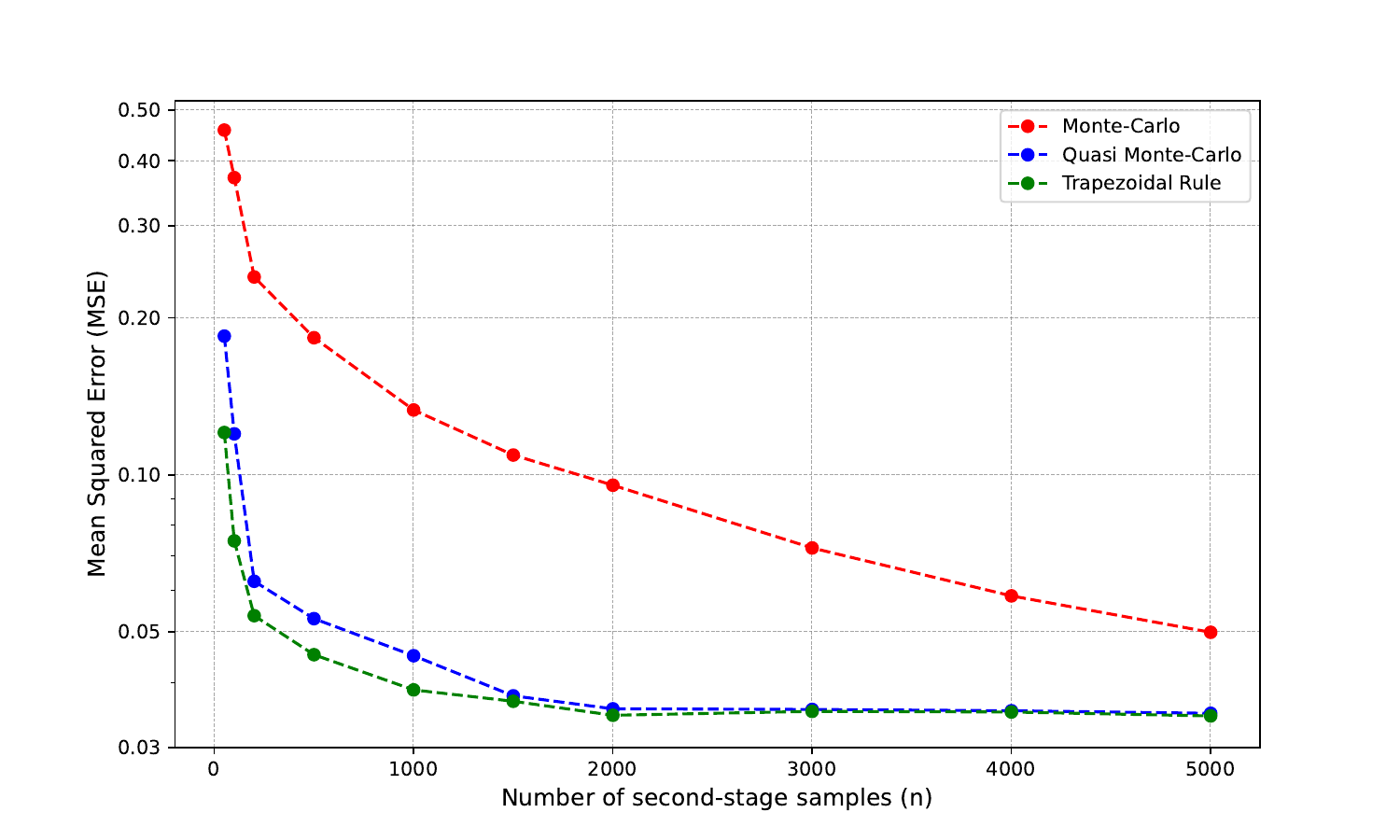}  
	\caption{{The test-set MSE vs. number of second-stage samples (n) in logarithmic scale.}}
	\label{figure_quadrature}
\end{figure}

{Finally, we analyze how different kernel quadrature methods affect the generalization performance of KEFNN. We keep the first-stage sample size at $m=2000$ and the noise variances fixed at $\sigma_1^2=3$ and $\sigma_2^2=0.05$. We evaluate various quadrature schemes, including Monte-Carlo methods, Quasi Monte-Carlo methods, and trapezoidal rules, across different second-stage sample sizes. In the case of Monte-Carlo methods, the discretization points $\{t_{i,j}\}_{j=1}^n$ are i.i.d.\ sampled from a uniform distribution on $[0,1]$. For the Quasi Monte-Carlo methods, we utilize low-discrepancy Sobol sequences for the discretization. The trapezoidal rule utilizes discretizations consisting of evenly spaced points on $[0,1]$. The weights $\{\theta_{i,j}\}_{j=1}^n$ of all three quadrature rules are determined by the composite trapezoidal rule. We plot the test-set MSE, averaged over five trials, against the number of second-stage samples $n$ for these three quadrature methods in Figure \ref{figure_quadrature}. The results indicate that Monte-Carlo method exhibits quite slow convergence, while the Quasi Monte-Carlo method converges more rapidly, although it is still slightly less effective than the evenly spaced discretizations.
}

\subsection{Examples on real functional datasets}

In this subsection, we apply the KEFNN to several classical small real functional datasets and evaluate its performance by comparing it with baseline dimension reduction techniques, specifically ``Raw data'', ``B-spline", and ``FPCA" methods. Each of these techniques is followed by a consistent deep neural network that consists of three hidden layers, each containing 32 neurons. We have opted not to include a comparison with AdaFNN, as the limited number of training samples in these datasets can lead to overfitting issues for AdaFNN in certain cases. To assess the performance of each method, we utilize RMSE (root mean square error) as our evaluation metric. The dimension reduction scores for ``B-spline" and ``FPCA" methods are computed using functions from the Python package ``scikit-fda" \citep{ramos2022scikit}. Furthermore, the selection of the number of B-spline bases in ``B-spline" approach, the number of eigenfunction bases in ``FPCA" approach, and the hyperparameters in our ``KEFNN" approach is conducted through cross-validation to ensure optimal model performance. The performance of these approaches across four learning tasks with real functional datasets is summarized in Table \ref{realdata}, presenting the mean and standard deviation of the test-set RMSE based on five training runs.

In \textit{Task 1}, we utilize the Medflies dataset from the Python package ``scikit-fda". This dataset comprises 534 samples of Mediterranean fruit flies (Medfly), documenting the daily number of eggs laid from day 5 to day 34 for each fly. Our objective is to leverage the early trajectory of daily egg-laying (the first 20 days) to predict the overall reproduction over a 30-day period. For \textit{Task 2} and \textit{Task 3}, we employ the Tecator dataset, also from the ``scikit-fda" package. This dataset consists of 240 meat samples, each represented by a 100-channel spectrum of absorbances along with the respective contents of moisture (water), fat, and protein. In Task 2, we aim to predict the fat content of a meat sample based on its absorbance spectrum, while in Task 3, we focus on predicting the moisture content. In \textit{Task 4}, we work with the Moisture dataset from the R package ``fds". This dataset features near-infrared reflectance spectra of 100 wheat samples, measured at 2 nm intervals from 1100 to 2500 nm, alongside their associated moisture content. Our task is to predict the moisture content of a wheat sample based on its near-infrared reflectance spectra. For all tasks, we follow a training, validation, and test split ratio of 64:16:20. In Task 1, the hyperparameters for KEFNN are set to $\gamma= 0.033$, $\beta=0.01$, $d_1= 100$. In the other three tasks, the hyperparameters are adjusted to $\gamma= 0.033$, $\beta=0.1$, $d_1= 150$.

\begin{table}
	\setlength{\tabcolsep}{6pt}
	\centering
	\caption{The test-set RMSE of the functional deep neural network employing various dimension reduction methods across four learning tasks with real functional datasets.} \label{realdata}
	\scalebox{0.94}{	\begin{tabular}{cccccc}
			\toprule  
			&Task 1  & 	Task 2  & 	Task 3   &  Task 4   \\
			\midrule  
			Raw data+NN    & 0.2394 $\pm$ 0.010   &  0.1505 $\pm$ 0.007   &  0.2076 $\pm$ 0.028 &  0.1268 $\pm$ 0.013	\\
			B-spline+NN    &  0.2421 $\pm$ 0.003   &  0.1288 $\pm$ 0.010   &  0.2059 $\pm$ 0.010 &  0.1238 $\pm$ 0.033	\\
			FPCA+NN    &  0.2550 $\pm$ 0.085   &  0.1121 $\pm$ 0.015   &  0.1255 $\pm$ 0.015 &  0.1763 $\pm$ 0.023	\\
			KEFNN    & \textbf{0.2254 $\pm$ 0.002}   &  \textbf{0.0614 $\pm$ 0.003}   &  \textbf{0.1003 $\pm$ 0.011}  &  \textbf{0.1163 $\pm$ 0.013}	\\
			\bottomrule 
	\end{tabular}}
\end{table}

From Table \ref{realdata}, it is evident that KEFNN consistently outperforms the other baseline dimension reduction methods across various real functional datasets. This result underscores the advantages and flexibility of our approach in learning various target functionals. In contrast, no single baseline method consistently outperforms the others across all learning tasks. Additionally, we observe that the standard deviation of the test-set RMSE for KEFNN is notably smaller than that of the other methods. This observation suggests that KEFNN might be less influenced by the randomness of data, leading to a more stable performance. Consequently, the generalization performance of KEFNN is likely to fall within a narrower range with higher probability, indicating enhanced reliability and consistency in its predictions.

\section*{Acknowledgments}
The research leading to these results received funding from the European Research Council under the European Union's Horizon 2020 research and innovation program/ERC Advanced Grant E-DUALITY (787960). This article reflects only the authors' views, and the EU is not liable for any use that may be made of the contained information; Flemish government (AI Research Program); Leuven.AI Institute. The research of Jun Fan is supported partially by the Research Grants Council of Hong Kong [Project No. HKBU 12302819] and [Project No. HKBU 12301619]. The work of D. X. Zhou described in this paper was fully/substantially/partially supported by the InnoHK initiative, The Government of the HKSAR, and the Laboratory for AI-Powered Financial Technologies. We also thank Dr. Zhen Zhang for his helpful discussions on this work.

\newpage

\appendix

\section*{Appendix}
This appendix provides proof of the main results in Section \ref{section2} and Section \ref{section4}, and some additional useful lemmas for the proof.

\section{Proof of Main Results} \label{section6}
In this section, we prove the main results in this paper.
\subsection{Proof of main results in Section \ref{section2}}

\subsubsection{Proof of Theorem \ref{approxrate1}}
{\bf Proof}. 
Denote $\Xi_{d_1}$ as the space spanned by the basis $\{\phi_1,\phi_2,\dots,\phi_{d_1}\}$, and the isometric isomorphism $U_{d_1}: (\RR^{d_1},\|\cdot\|_2) \to (\Xi_{d_1},\|\cdot\|_{L_2(\mu)})$ as
\begin{equation}
	U_{d_1}(v)= \sum_{i=1}^{d_1} v_i \phi_i,  \qquad v\in \RR^{d_1},
\end{equation}
where the measure $\mu$ is specified later in the proof.

To prove the main results of approximation rates for the target functional $F$, the key is the following error decomposition with three error terms to bound.
\begin{equation}
	\begin{aligned}
		& \sup_{f\in \mathcal{F}} |F(f)-F_{NN}(f)| \\
		\leq &  \sup_{f\in \mathcal{F}} |F(f)-F(L_K f)| \\
		+ & \sup_{f\in \mathcal{F}} |F(L_K f)-F(U_{d_1} \circ T(L_K f))| \\
		+ & \sup_{f\in \mathcal{F}} |F \circ U_{d_1} ( T(L_K f))-F_{NN}(f)|,
	\end{aligned}
\end{equation}

The first term focuses on the error of approximating the input function $f$ by its embedded function $L_K f$, the second term focuses on the error of approximating the embedded function $L_K f$ by its projection on the subspace spanned by the eigenfunction basis, and the final term focuses on the error of approximating a $\lambda$-H\"older continuous function by a deep ReLU neural network. We then bound these three error terms individually in the following.

For the first term, with the choice of the embedding kernel $K$ being (\ref{embedkernel1}), according to Lemma \ref{steinwart}, $\forall f \in \mathcal{F}$, we have
\begin{equation*}
	\|f- L_K f\|_{L_2(\RR^d)} \leq C_{d,r}  \omega_{r,L_2(\RR^d)}\left(f,\frac{\gamma}{\sqrt{2}} \right) \leq C_{d,r,\alpha} \gamma^{\alpha},
\end{equation*}
where $C_{d,r,\alpha}= C_{d,r} (\sqrt{2})^{-\alpha}$ is a constant only depending on $d,r,\alpha$. Thus,
\begin{equation} \label{approerror1}
	\sup_{f\in \mathcal{F}} |F(f)-F(L_K f)|  \leq  \sup_{f\in \mathcal{F}}  C_F \|f-L_K f\|_{L_2(\RR^d)}^{\lambda} 
	\leq  C_F C_{d,r,\alpha}^\lambda \gamma^{\alpha\lambda}.
\end{equation}

For the second term, since $\forall f\in \mathcal{F}$, $L_K f \in \mathcal{H}_{\gamma}$, it can be written as $L_K f= \sum_{i=1}^\infty a_i \sqrt{\lambda_i} \phi_i$, with $\|L_K f\|_{H_\gamma}= \sqrt{\sum_{i=1}^\infty a_i^2} \leq (2^r-1)\pi^{-\frac{d}{4}} \gamma^{-\frac{d}{2}}$  by Lemma \ref{steinwart}.
Then we have
\begin{equation*}
	T(L_K f)= \left[a_1 \sqrt{\lambda_1}, a_2 \sqrt{\lambda_2}, \dots, a_{d_1} \sqrt{\lambda_{d_1}} \right],
\end{equation*}
and hence
\begin{equation*}
	U_{d_1} \circ T(L_K f)= \sum_{i=1}^{d_1}  a_i \sqrt{\lambda_i} \phi_i.
\end{equation*}
It follows that
\begin{eqnarray*}
	\nonumber && \left\|L_K f- U_{d_1} \circ T(L_K f)\right\|_{L_2(\mu)}^2 = \left\| \sum_{i=d_1+1}^{\infty}  a_i \sqrt{\lambda_i} \phi_i \right\|_{L_2(\mu)}^2 \\
	& = & \sum_{i=d_1+1}^{\infty}  a_i^2 \lambda_i \leq \lambda_{d_1+1} \sum_{i=d_1+1}^{\infty}  a_i^2 \leq \lambda_{d_1+1} \|L_K f\|_{H_\gamma}^2.
\end{eqnarray*}

The eigenvalues of the integral operator associated with a one-dimensional Gaussian kernel can be expressed with an explicit formulation that exhibits exponential decay \citep[pp. 28]{Fasshauer2011}, specifically
\begin{equation*}
	\lambda_i= \frac{\beta}{\gamma^{2i} \left(\frac{\beta^2}{2} \left(1+\sqrt{1+\left(\frac{2}{\beta \gamma} \right)^2} \right)+\frac{1}{\gamma^2} \right)^{i+\frac{1}{2}}} \leq \frac{\beta \gamma}{\left(1+\frac{\beta \gamma}{2} \right)^{2i+1}},
\end{equation*}
where the measure $\mu$ is selected as a Gaussian distribution, characterized by the density function $u(x)= \frac{\beta}{\sqrt{\pi}} \exp^{-\beta^2 x^2}$, and $\beta$ serves as a global scaling parameter. We use this explicit form because the bandwidth $\gamma$ is a hyperparameter that requires careful tuning later and should not be treated as a constant. 

Notice that the $\binom{j+d}{d}$-th eigenvalue of the integral operator associated with a multi-dimensional Gaussian kernel $k_\gamma$ can be bounded by $\frac{(\beta \gamma)^d}{\left(1+\frac{\beta \gamma}{2} \right)^{2j+d}}$ \citep[pp. 98]{rasmussen2006gaussian}. If we denote $\binom{j+d}{d} \leq  \frac{(j+\frac{d}{2})^d}{d!}=: d_1$, then the $d_1$-th eigenvalue of the integral operator corresponding to $k_\gamma$ can be bounded by
\begin{equation} \label{gaussiankerneleigen}
	\lambda_{d_1} \leq \frac{(\beta \gamma)^d}{\left(1+\frac{\beta \gamma}{2} \right)^{2c_2 d_1^{\frac{1}{d}}}},
\end{equation}
where $c_2=(d!)^\frac{1}{d}> \frac{d}{e}$ by Stirling's theorem. 
By choosing $\beta=2$, we get
\begin{eqnarray} \label{approerror2}
	\nonumber \sup_{f\in \mathcal{F}} |F(L_K f)-F(U_{d_1} \circ T(L_K f))| &\leq& \sup_{f\in \mathcal{F}} C_F \|L_K f-U_{d_1} \circ T(L_K f)\|_{L_2(\mu)}^\lambda \\
	&\leq& C_F C_{d,r,\lambda}  \left(1+\gamma \right)^{-c_2\lambda d_1^{\frac{1}{d}}},
\end{eqnarray}
where $C_{d,r,\lambda}$ is a constant depending only on $d,r,\lambda$.

For the third term, since $F\circ U_{d_1}$ follows the smoothness of $F$, $\forall y_1,y_2 \in \RR^{d_1}$, we have
\begin{equation}
	\begin{aligned}
		|F\circ U_{d_1}(y_1)-F\circ U_{d_1}(y_2)| &\leq C_F \|U_{d_1}(y_1)-U_{d_1}(y_2)\|_{L_2(\mu)}^\lambda \\
		&= C_F \|y_1-y_2\|_2^\lambda.
	\end{aligned}
\end{equation}
Therefore, $F\circ U_{d_1}$ is essentially a $\lambda$-H\"older continuous function defined on $[-R,R]^{d_1}$, where $R= \| T(L_K f)\|_\infty$. Then by Lemma \ref{lipschitzappro},
there exists a deep ReLU neural network  $H_{NN}$ with depth $d_1^2+d_1+1$,  and $M$ nonzero parameters such that
\begin{equation*}
	\sup_{y\in [-R,R]^{d_1}} |F\circ U_{d_1} ( y)-H_{NN}(y)| \leq 2C_F d_1 \left(\frac{c_0 d_1^{\frac{4}{d_1}} R}{M^{\frac{1}{d_1}}} \right)^\lambda,
\end{equation*}
where $c_0$ is a constant independent of $d_1, M$.

Recalling the architecture of our functional network as defined in Definition \ref{deffunctionalNN}, it can be alternatively represented as $F_{NN}(f)=H_{NN}\circ T(L_K f)$, where $H_{NN}$ denotes a standard deep ReLU neural network with an input dimension $d_1$. Moreover, according to Lemma \ref{steinwart} and the fact that $C_K = \tilde{c}_0 \gamma^{-d}$ with $\tilde c_0= \pi^{-\frac{d}{2}} \sum_{j=1}^r \binom{r}{j} \frac{1}{j^d}$ being a positive constant, the radius of the input cube can be bounded by
\begin{equation*}
	\begin{aligned}
		R &= \| T(L_K f)\|_\infty \leq \|L_K f\|_{L_2(\mu)} \leq  \sqrt{\mu(\RR^d)} \|L_K f\|_\infty \\
		& \leq  \sqrt{\mu(\RR^d) C_K} \|L_K f\|_{H_\gamma}  \leq \sqrt{\tilde{c}_0 \mu(\RR^d)} (2^r-1)\pi^{-\frac{d}{4}}  \gamma^{-d} ,
	\end{aligned}
\end{equation*}
Hence, we conclude that there exists a functional net $F_{NN}$ following the architecture in Definition \ref{deffunctionalNN} with depth $J=d_1^2+d_1+2$ and $M$ nonzero parameters such that
\begin{equation} \label{approerror3}
	\sup_{f\in \mathcal{F}} |F\circ U_{d_1} ( T(L_K f))-F_{NN}(f)| \leq 2C_F d_1 \left(\frac{c_1 d_1^{\frac{4}{d_1}} \gamma^{-d}}{M^{\frac{1}{d_1}}} \right)^\lambda,
\end{equation}
where $c_1= c_0 \sqrt{\tilde{c}_0 \mu(\RR^d)} (2^r-1) \pi^{-\frac{d}{4}}$ is a constant independent of $d_1, M$.

Finally, by combining (\ref{approerror1}), (\ref{approerror2}), and (\ref{approerror3}), we have
\begin{equation}
	\sup_{f\in \mathcal{F}} |F(f)-F_{NN}(f)|  \leq C_F C_{d,r,\alpha}^\lambda \gamma^{\alpha\lambda}+C_F C_{d,r,\lambda}  \left(1+\gamma \right)^{-c_2\lambda d_1^{\frac{1}{d}}}+ 2C_F d_1 \left(\frac{c_1 d_1^{\frac{4}{d_1}} \gamma^{-d}}{M^{\frac{1}{d_1}}} \right)^\lambda.
\end{equation}
To balance these three error terms, we can choose
\begin{equation}
	d_1 = \tilde c_1 \frac{\log M}{\log \log M}, \quad \gamma = \tilde c_2 \left(\frac{\log \log M}{\log M}\right)^{\frac{1}{d}} \log\log M,
\end{equation}
where $\tilde c_1, \tilde c_2$ are positive constants that will be determined later. Thus, considering the first error term, we have
\begin{equation} \label{th1error1}
	\gamma^{\alpha\lambda}= \tilde {c}_2^{\alpha\lambda} \left(\log M\right)^{-\frac{\alpha \lambda}{d}} \left(\log\log M\right)^{\left(\frac{1}{d}+1\right) \alpha\lambda},
\end{equation}
since $(1+\frac{1}{x})^x \geq c_3$ for some constant $c_3 \in (1,e)$ when $x$ is sufficiently large, we can set $x= \frac{1}{\gamma}$. When $M$ is large enough, considering the second error term, we have
\begin{equation} \label{th1error2}
	\left(1+\gamma \right)^{-c_2\lambda d_1^{\frac{1}{d}}}= \left[(1+\frac{1}{x})^{x}\right]^{-c_2 \tilde c_1^{\frac{1}{d}}\tilde c_2 \lambda \log\log M} 
	\leq \left(\log M \right)^{- c_2 \tilde{c}_1^{\frac{1}{d}} \tilde c_2 \lambda \log c_3}.
\end{equation}
Finally, since $d_1^{\frac{4}{d_1}} \leq c_4$ for some constant $c_4$, and given that $M^{\frac{1}{d_1}}= (\log M)^{\frac{1}{\tilde{c}_1}}$. Considering the third error term, we have
\begin{equation} \label{th1error3}
	d_1 \left(\frac{c_1 d_1^{\frac{4}{d_1}} \gamma^{-d}}{M^{\frac{1}{d_1}}} \right)^\lambda \leq c_1^\lambda c_4^\lambda \tilde{c}_1 \tilde{c}_2^{-d} \left(\log M \right)^{\frac{3\lambda}{2}+1-\frac{\lambda}{\tilde{c}_1}}  \left(\log\log M \right)^{-1-(d+1)\lambda}.
\end{equation}
Therefore, to balance the error terms \eqref{th1error1}, \eqref{th1error2}, and \eqref{th1error3} w.r.t. the exponential rates on $\log M$, we can choose the constants to satisfy
\begin{equation}
	\tilde{c}_1 < \frac{1}{1+\frac{\alpha}{d}+\frac{1}{\lambda}}, \qquad \tilde{c}_2 > \frac{\alpha}{dc_2 \tilde{c}_1^{\frac{1}{d}}\log c_3}.
\end{equation}
Thus we obtain the desired bounds on the approximation error and the depth of the functional network, with the constants specified as $\tilde{c}_3= 2\tilde{c}_1^2 +2$, and $\tilde{c}_4= C_F C_{d,r,\alpha}^\lambda \tilde {c}_2^{\alpha\lambda}+ C_F C_{d,r,\lambda}+ 2C_F c_1^\lambda c_4^\lambda \tilde{c}_1 \tilde{c}_2^{-d\lambda}$. 	\qed

\subsubsection{Proof of Theorem \ref{approxrate12}}

{\bf Proof}. 
The only difference between the functional network architecture used in the proof here and that presented in Theorem \ref{approxrate1} lies in the kernel embedding step, where the domain of the input function is changed from $\RR^d$ in Theorem \ref{approxrate1} to $\Omega$ considered here. By denoting
\begin{equation*}
	L_K f(x)= \int_{\Omega} K(x,t)f(t) dt, \quad x\in \Omega,
\end{equation*}
and $\tilde{f}$ as the extension by zero of $f$ to $\RR^d$. We have
\begin{equation*}
	L_K \tilde{f}(x)= \int_{\RR^d} K(x,t) \tilde f(t) dt,  \quad x\in \RR^d.
\end{equation*}
Notice that $L_K f=(L_K \tilde{f})|_{\Omega}$, it follows that
\begin{equation*}
	\left\|f-L_K f\right\|_{L_2(\mu)} \leq \left\|\tilde f-L_K \tilde{f}\right\|_{L_2(\tilde \mu)},
\end{equation*}
where we set $\mu=\tilde{\mu}|_\Omega$, and $\tilde{\mu}$ is chosen as the same Gaussian distribution in the proof of Theorem \ref{approxrate1}. The rest of the proof follows directly from the result in Theorem \ref{approxrate1} applied to approximate nonlinear functionals defined on $\tilde{f}$.	\qed

\subsubsection{Proof of Theorem \ref{approxrate2}}

{\bf Proof}. 
In this case, the domain of the target functional is already a Gaussian RKHS. As a result, we can use the error decomposition with only the last two error terms considered in Theorem \ref{approxrate1}, with a focus on balancing the choice of $d_1$. Specifically, by denote the mapping $V_{d_1}: (\RR^{d_1},|\cdot|) \to (\Xi_{d_1},\|\cdot\|_{L_2(\mu)})$ as
\begin{equation}
	V_{d_1}(v)= \sum_{i=1}^{d_1} \frac{v_i}{\lambda_i} \phi_i,  \qquad v\in \RR^{d_1}.
\end{equation}
Now the error decomposition contains only two error terms,
\begin{equation}
	\begin{aligned}
		& \sup_{f\in \mathcal{F}} |F(f)-F_{NN}(f)| \\
		\leq & \sup_{f\in \mathcal{F}} |F(f)-F(V_{d_1} \circ T(L_K f))| \\
		+ & \sup_{f\in \mathcal{F}} |F \circ V_{d_1} ( T(L_K f))-F_{NN}(f)|.
	\end{aligned}
\end{equation}

Let us consider the first error term. For any $f\in \mathcal{H}_\gamma$, denote $f=\sum_{i=1}^\infty a_i \sqrt{\lambda_i} \phi_i$, with $\|f\|_{H_\gamma}= \sqrt{\sum_{i=1}^\infty a_i^2}\leq 1$. Since we choose the embedding kernel $K$ as defined in \eqref{embeddingkernelg}, the projection kernel $K_0$ as $k_\gamma$, and the projection bases as the first $d_1$ eigenfunctions of the integral operator $L_{K_0}^\mu$, we have
\begin{equation*}
	T(L_K f)= \Big[a_1 \lambda_1^{\frac{3}{2}}, a_2  \lambda_2^{\frac{3}{2}}, \dots, a_{d_1}  \lambda_{d_1}^{\frac{3}{2}} \Big],
\end{equation*}
\begin{equation*}
	V_{d_1} \circ T(L_K f)= \sum_{i=1}^{d_1}  a_i  \sqrt{\lambda_i} \phi_i.
\end{equation*}
It follows that
\begin{eqnarray*}
	\nonumber && \left\|f- V_{d_1} \circ T(L_K f)\right\|_{L_2(\mu)}^2 = \Big\| \sum_{i=d_1+1}^{\infty}  a_i \sqrt{\lambda_i} \phi_i \Big\|_{L_2(\mu)}^2 \\
	& = & \sum_{i=d_1+1}^{\infty}  a_i^2 \lambda_i \leq \lambda_{d_1+1} \sum_{i=d_1+1}^{\infty}  a_i^2 \leq \lambda_{d_1+1} \| f\|_{H_\gamma}^2.
\end{eqnarray*}
Therefore, by utilizing the eigenvalue upper bound of the Gaussian kernel as specified in \eqref{gaussiankerneleigen}, the first error term is bounded by
\begin{eqnarray} \label{approerror4}
	\nonumber \sup_{f\in \mathcal{F}} |F(f)-F(V_{d_1} \circ T(L_K f))| &\leq& \sup_{f\in \mathcal{F}} C_F \|f-V_{d_1} \circ T(L_K f)\|_{L_2(\mu)}^\lambda \\
	&\leq& C_F C_{d,\gamma,\lambda}  e^{-c_2\lambda d_1^{\frac{1}{d}}},
\end{eqnarray}
where $C_{d,\gamma,\lambda}$ is a constant depending only on $d,\gamma,\lambda$, and $c_2$ is a constant with $c_2> \frac{d}{e}$.

As for the second error term, notice that $F\circ V_{d_1}$ follows the smoothness of $F$ with a scaling on the constant term, that is
\begin{equation}
	\begin{aligned}
		|F\circ V_{d_1}(y_1)-F\circ V_{d_1}(y_2)| &\leq C_F \|V_{d_1}(y_1)-V_{d_1}(y_2)\|_{L_2(\mu)}^\lambda \\
		&\leq \frac{C_F}{\lambda_{d_1}^\lambda} \|y_1-y_2\|_2^\lambda.
	\end{aligned}
\end{equation}
Furthermore, the radius of the input for the deep ReLU neural network can be bounded by
\begin{equation*}
	\begin{aligned}
		R& =\|T(L_K f)\|_\infty \leq \|L_K f\|_{L_2{(\mu)}} \leq \sqrt{\mu(\Omega)} \|L_K f\|_{L_\infty(\Omega)} \\
		&= \sqrt{\mu(\Omega)} \sup_{x\in \Omega} \int_{\Omega} K(x,t) f(t) dt  \leq \sqrt{\mu(\Omega)} C_K \|f\|_{L_2(\Omega)} \leq \sqrt{\mu(\Omega)} C_K,
	\end{aligned}
\end{equation*}
Therefore, similar to the proof of Theorem \ref{approxrate1}, we can apply Lemma \ref{lipschitzappro} to establish the existence of a functional network $F_{NN}=H_{NN}\circ T(L_K f)$ in the format defined by Definition \ref{deffunctionalNN}, with a depth of $J=d_1^2+d_1+2$ and $M$ nonzero parameters, such that
\begin{eqnarray} \label{approerror5}
	\nonumber \sup_{f\in \mathcal{F}} |F\circ V_{d_1} ( T(L_K f))-F_{NN}(f)| & \leq & 2C_F\frac{d_1}{\lambda_{d_1}^\lambda} \left(\frac{c_0 R d_1^{\frac{4}{d_1}} }{M^{\frac{1}{d_1}}} \right)^\lambda \\
	& \leq & c_5 d_1 e^{2c_2 \lambda d_1^{\frac{1}{d}}} \left(\frac{d_1^{\frac{4}{d_1}} }{M^{\frac{1}{d_1}}} \right)^\lambda, 
\end{eqnarray}
where for the second inequality we utilize the eigenvalue lower bound of the Gaussian kernel, and $c_5= 2C_F c_0^\lambda \mu(\Omega)^{\frac{\lambda}{2}} (\frac{\pi \gamma^2}{2})^{\frac{d\lambda}{4}}$ is a positive constant.

Finally, by combining (\ref{approerror4}) and (\ref{approerror5}), we have
\begin{equation}
	\sup_{f\in \mathcal{F}} |F(f)-F_{NN}(f)|  \leq C_F C_{d,\gamma,\lambda}  e^{-c_2\lambda d_1^{\frac{1}{d}}}+ c_5 d_1 e^{2c_2 \lambda d_1^{\frac{1}{d}}} \left(\frac{d_1^{\frac{4}{d_1}} }{M^{\frac{1}{d_1}}} \right)^\lambda.
\end{equation}
Furthermore, by choosing
\begin{equation}
	d_1 = c_6 \left(\log M \right)^{\frac{d}{d+1}},
\end{equation}
where $c_6$ is a constant to be determined later. Considering the first error term, we have
\begin{equation} \label{th3error1}
	e^{-c_2\lambda d_1^{\frac{1}{d}}}= e^{-c_2 c_6^{\frac{1}{d}}\lambda  \left(\log M \right)^{\frac{1}{d+1}}}.
\end{equation}
Considering the second error term, we have
\begin{equation} \label{th3error2}
	d_1 e^{2c_2 \lambda d_1^{\frac{1}{d}}} \left(\frac{d_1^{\frac{4}{d_1}} }{M^{\frac{1}{d_1}}} \right)^\lambda \leq c_6 c_4^\lambda e^{\left(2c_2 c_6^{\frac{1}{d}}-\frac{1}{c_6}\right) \lambda  \left(\log M \right)^{\frac{1}{d+1}}} \left(\log M \right)^{\frac{d}{d+1}}.
\end{equation}
Therefore, we can choose $c_6 = (3c_2)^{-\frac{d}{d+1}}$ to balance the dominated terms specified in \eqref{th3error1} and \eqref{th3error2}. It follows that
\begin{equation}
	\sup_{f\in \mathcal{F}} |F(f)-F_{NN}(f)| \leq \tilde{c}_6 e^{-c_7 \lambda \left(\log M \right)^{\frac{1}{d+1}}} \left(\log M\right)^{\frac{d}{d+1}},
\end{equation}
where $\tilde{c}_6= C_F C_{d,\gamma,\lambda}+ c_5 c_6 c_4^\lambda$, and
\begin{equation}
	c_7= 3^{-\frac{1}{d+1}} c_2^{\frac{d}{d+1}} > 3^{-\frac{1}{d+1}} \left(\frac{d}{e} \right)^{\frac{d}{d+1}}> \left(\frac{d}{3e}\right)^{\frac{d}{d+1}}.
\end{equation}
Thus we complete the proof, with the constant $\tilde{c}_5 =2  c_6^2+ 2$.	\qed

\subsubsection{Proof of Theorem \ref{approxrate3}}

{\bf Proof}. 
The proof essentially follows the framework of the proof for Theorem \ref{approxrate2}, with the only difference being the rates of eigenvalue decay for the integral operator $L^\mu_{K_\alpha}$. The upper bounds for the eigenvalues were established in \cite[pp. 30]{bach2017equivalence}, and we can obtain the lower bounds using the same approach outlined there. Consequently, we have
\begin{equation}
	C_1 (\log k)^{2\alpha} k^{-2\alpha} \leq \lambda_k \leq C_2(\log k)^{2\alpha(d-1)} k^{-2\alpha},
\end{equation}
where $C_1, C_2$ are positive constants. According to the same error decomposition in the proof of Theorem \ref{approxrate2}, there exists a functional net $F_{NN}$ in the form of Definition \ref{deffunctionalNN} with depth $J=d_1^2+d_1+2$ and number of nonzero parameters $M$ such that
\begin{equation}
	\sup_{f\in \mathcal{F}} |F(f)-F_{NN}(f)|  \leq C_F C_2^{\frac{\lambda}{2}} (\log d_1)^{\alpha \lambda(d-1)} {d_1}^{-\alpha\lambda}+ C_3 d_1^{2\alpha\lambda+1}  (\log d_1)^{-2\alpha\lambda} \left(\frac{d_1^{\frac{4}{d_1}} }{M^{\frac{1}{d_1}}} \right)^\lambda,
\end{equation}
where $C_3$ is a positive constant independent of $d_1,M$. Finally, by choosing
\begin{equation}
	d_1 = C_4 \frac{\log M}{\log\log M},
\end{equation}
where $C_4$ is a constant to be determined later. Considering the first error term, we have
\begin{equation} \label{th4error1}
	(\log d_1)^{\alpha \lambda(d-1)} {d_1}^{-\alpha\lambda}\leq C_4^{-\alpha\lambda} \left(\log M\right)^{-\alpha\lambda} \left(\log\log M\right)^{\alpha\lambda(d-2)}.
\end{equation}
Considering the second error term, we have
\begin{equation} \label{th4error2}
	d_1^{2\alpha\lambda+1} \left(\frac{d_1^{\frac{4}{d_1}} }{M^{\frac{1}{d_1}}} \right)^\lambda\leq c_4^\lambda C_4^{2\alpha\lambda+1} \left(\log M\right)^{2\alpha\lambda+1-\frac{\lambda}{C_4}} \left(\log \log M\right)^{-2\alpha\lambda-1}.
\end{equation}
Therefore, we can choose $C_4 \leq \frac{\lambda}{3\alpha\lambda+1}$ to balance the dominated $\log M$ terms in \eqref{th4error1} and \eqref{th4error2}.
Then we can obtain the desired bounds on the depth of functional network and the approximation error, with the constants specified as $C_5= 2C_4^2+2$, and $C_6= C_F C_2^{\frac{\lambda}{2}}C_4^{-\alpha\lambda}+ C_3 c_4^\lambda C_4^{2\alpha\lambda+1}$.	\qed

\subsection{Proof of main results in Section \ref{section4}}

\subsubsection{Proof of Proposition \ref{raodong}}

{\bf Proof}. 
\begin{equation*}
	\begin{aligned}
		|h(x)-h(y)|&\le \left| \sum\limits_{i=1}^{(N+1)^{d_1}}c_i\psi\left(\frac{N}{2R}(x-b_i)\right)-\sum\limits_{i=1}^{(N+1)^{d_1}}c_i\psi\left(\frac{N}{2R}(y-b_i)\right)\right|\\
		&\le \sum\limits_{i=1}^{(N+1)^{d_1}}L \left|\psi\left(\frac{N}{2R}(x-b_i)\right)-\psi\left(\frac{N}{2R}(y-b_i)\right)\right|.
	\end{aligned}
\end{equation*}
According to Lemma \ref{lemma_min}, we have
\begin{equation*}
	\begin{aligned}
		|\psi(x)-\psi(y)|&\le |\min\big\{\min_{k\neq j}(1+x_k-x_j),\min_{k}(1+x_k),\min_{k}(1-x_k)\big\}
		\\
		&\quad-\min\big\{\min_{k\neq j}(1+y_k-y_j),\min_{k}(1+y_k),\min_{k}(1-y_k)\big\} |\\
		&\le 2(d_1^2+d_1)||x-y||_1.
	\end{aligned}
\end{equation*}
Therefore we have
\begin{equation*}
	\begin{aligned}
		|h(x)-h(y)|&\le\sum\limits_{i=1}^{(N+1)^{d_1}}2(d_1^2+d_1) L\left|\left|\frac{N}{2R}(x-y)\right|\right|_1\\
		&\le \frac{L}{R}N(N+1)^{d_1} (d_1^2+d_1)||x-y||_1,
	\end{aligned}
\end{equation*}
which completes the proof by replacing $R$ with $\sqrt{C_\mu} C_K$.	\qed

\subsubsection{Proof of Proposition  \ref{errordec}}

{\bf Proof}. 
Note that 
\begin{equation*}
	\begin{aligned}
		\mathcal{E}\left(\pi_{L}F_{\widehat D} \right)- \mathcal{E}\left(F_\rho \right)&= I_1+ I_2 + 2\left\{\mathcal{E}_D\left(\pi_{L}F_{\widehat D} \right) - \mathcal{E}_{D}\left(\pi_{L} F_{D} \right)\right\}\\
		&\le I_1+I_2+2\left\{\mathcal{E}_D\left(\pi_{L}F_{\widehat D} \right) - \mathcal{E}_{D}\left(\pi_{L} F_{D} \right)\right\}\\
		&\quad+2\left\{\mathcal{E}_{\widehat D}\left(\pi_{L} H_{D} \circ \widehat{L_K} \right)-\mathcal{E}_{\widehat D}\left(\pi_{L}H_{\widehat D} \circ \widehat{L_K} \right)\right\}.
	\end{aligned}
\end{equation*}
The inequality from above holds since $$\mathcal{E}_{\widehat D}\left(\pi_{L} H_{D} \circ \widehat{L_K} \right)\ge\mathcal{E}_{\widehat D}\left(H_{\widehat D} \circ \widehat{L_K} \right)\ge\mathcal{E}_{\widehat D}\left(\pi_{L}H_{\widehat D} \circ \widehat{L_K} \right).$$
The desired result is achieved by noting
\begin{equation*}
	\begin{aligned}
		&2\left\{\mathcal{E}_D\left(\pi_{L}F_{\widehat D} \right) - \mathcal{E}_{D}\left(\pi_{L} F_{D} \right)\right\}
		+2\left\{\mathcal{E}_{\widehat D}\left(\pi_{L} H_{D} \circ \widehat{L_K} \right)-\mathcal{E}_{\widehat D}\left(\pi_{L}H_{\widehat D} \circ \widehat{L_K} \right)\right\}\\
		& \leq 2 \left| \frac{1}{m}\sum_{i=1}^{m} \left(\pi_L H_{\widehat D}(u_i)-y_i \right)^2 -  \frac{1}{m}\sum_{i=1}^{m} \left(\pi_L H_{\widehat D}(\hat u_i)-y_i \right)^2 \right|   \\
		& + 2 \left| \frac{1}{m}\sum_{i=1}^{m} \left(\pi_L H_{D}(u_i)-y_i \right)^2 -  \frac{1}{m}\sum_{i=1}^{m} \left(\pi_L H_{D}(\hat u_i)-y_i \right)^2 \right|  \\
		&\le 16L\sup_{H\in \mathcal{H}}\frac{1}{m}\sum_{i=1}^{m}\left|\pi_L H(u_i)- \pi_L H(\hat{u}_i)\right| \\
		&\le 16L\sup_{H\in \mathcal{H}}\frac{1}{m}\sum_{i=1}^{m}\left|H(u_i)-H(\hat{u}_i)\right|=I_3,
	\end{aligned}
\end{equation*}
where the last inequality above is from $|\pi_{L}H(u_i)-\pi_{L}H(\hat{u}_i)|\le |H(u_i)-H(\hat{u}_i)|$.	\qed

\subsubsection{Proof of Theorem \ref{oracleinequality}}
{
	In the proof, we will utilize the covering number as the tool of the complexity measure of a hypothesis space $\mathcal{H}$. Let $f_1^m= (f_1,\dots,f_m)$ be $m$ fixed functions in the input function space $\mathcal{F}$. Let $\nu_m$ be the corresponding empirical measure, i.e.,
	\begin{equation}
		\nu_m(A)= \frac{1}{m} \sum_{i=1}^m I_A(f_i), \qquad A \subseteq \mathcal{F}.
	\end{equation}
	Then
	\begin{equation}
		\| F \|_{L_p(\nu_m)}= \left\{ \frac{1}{m} \sum_{i=1}^m |F(f_i)|^p \right\}^{\frac{1}{p}},
	\end{equation}
	and any $\epsilon$-cover of $\mathcal{H}$ w.r.t. $\|\cdot\|_{L_p(\nu_m)}$ is called a $L_p$ $\epsilon$-cover of $\mathcal{H}$ on $f_1^m$, the $\epsilon$-covering number of $\mathcal{H}$ w.r.t. $\|\cdot\|_{L_p(\nu_m)}$ is denoted by
	$$\mathcal{N}_p(\epsilon,\mathcal{H},f_1^m),$$
	which is the minimal integer $N$ such that there exist functionals $F_1,\dots,F_N: \mathcal{F} \to \RR$ with the property that for every $F\in \mathcal{F}$, there is a $j=j(F)\in \{1,\dots,N\}$ such that
	\begin{equation*}
		\left\{ \frac{1}{m} \sum_{i=1}^m |F(f_i)-F_j(f_i)|^p \right\}^{\frac{1}{p}} < \epsilon.
	\end{equation*}
	Moreover, we denote $\mathcal{M}_p(\epsilon,\mathcal{H},f_1^m)$ as the $\epsilon$-packing number of $\mathcal{H}$ w.r.t. $\|\cdot\|_{L_p(\nu_m)}$, which is the largest integer $N$ such that there exist functionals $F_1,\dots,F_N: \mathcal{F} \to \RR$ satisfying $\|F_j-F_k\|_{L_p(\nu_m)}\geq \epsilon$ for all $1\leq j<k\leq N$. 
} 

{\bf Proof}. 
From Proposition \ref{errordec}, we know it suffices to bound the expectation of $I_1$, $I_2$ and $I_3$ separately.
\begin{itemize}
	\item First, we  derive a bound for $I_1$ in a probability form. Denote $\pi_L\mathcal{H}\circ L_K= \{\pi_L H\circ L_K: H\in \mathcal{H}\}$. For any $\epsilon>0$,
	\begin{equation*}
		\begin{aligned}
			&P\left\{I_1>\epsilon\right\}\\
			=& P\left\{||\pi_LF_{\widehat D}-F_{\rho}||_{\rho}^2-\left(\mathcal{E}_{D}(\pi_{L}F_{\widehat D})-\mathcal{E}_{D}({F_{\rho}})\right)>\frac{1}{2}\left(\epsilon+||\pi_LF_{\widehat D}-F_{\rho}||_{\rho}^2\right)\right\}\\
			\le & P\bigg\{\exists F\in \pi_L\mathcal{H}\circ L_K: ||F-F_{\rho}||_{\rho}^2-\left(\mathcal{E}_{D}(F)-\mathcal{E}_{D}({F_{\rho}})\right) \\
			& \quad > \frac{1}{2}\left(\frac{\epsilon}{2}+\frac{\epsilon}{2}+||F-F_{\rho}||_{\rho}^2\right)\bigg\}\\
			\le &  14\sup_{f_1^m}\mathcal{N}_1\left(\frac{\epsilon}{80L},\pi_L\mathcal{H}\circ L_K,f_1^m\right)\exp\left(-\frac{m\epsilon}{5136L^4}\right),
		\end{aligned}
	\end{equation*} 
	where we have used Lemma \ref{lemma_theorem9.14} in deriving the last inequality with $\alpha=\beta= \frac{\epsilon}{2}$, and $\delta= \frac{1}{2}$. Therefore, for any $a\ge \frac{1}{m}$, 
	\begin{equation*}
		\begin{aligned}
			EI_1&\le \int_{0}^{\infty}P\{I_1>u\}du\le a+\int_{a}^{\infty}P\{I_1>u\}du\\
			&\le a+\int_{a}^{\infty}14\sup_{f_1^m}\mathcal{N}_1\left(\frac{1}{80Lm},\pi_L\mathcal{H}\circ L_K,f_1^m\right)\exp\left(-\frac{mu}{5136L^4}\right)du\\
			&\le a+	14\sup_{f_1^m}\mathcal{N}_1\left(\frac{1}{80Lm},\pi_L\mathcal{H}\circ L_K,f_1^m\right)\frac{5136L^4}{m}\exp\left(-\frac{ma}{5136L^4}\right),
		\end{aligned}
	\end{equation*}
	which is minimized if we choose
	\begin{equation*}
		a=\frac{5136L^4}{m}\log\left(14\sup_{f_1^m}\mathcal{N}_1\left(\frac{1}{80Lm},\pi_L\mathcal{H}\circ L_K,f_1^m\right)\right),
	\end{equation*}
	therefore, we have
	\begin{equation}\label{EI1}
		\begin{aligned}
			EI_1&\le \frac{5136L^4}{m}\left\{\log\left(14\sup_{f_1^m}\mathcal{N}_1\left(\frac{1}{80Lm},\pi_L\mathcal{H}\circ L_K,f_1^m\right)\right)+1\right\}\\
			&\le \frac{5136L^4}{m}\left\{\log(28)+2c_2'JM\log(M)\log(320eL^2m)+1\right\}\\
			&\le \frac{c'_1 JM\log M\log m}{m}, 
		\end{aligned}
	\end{equation}
	where the second inequality above is from Lemma \ref{coveringnumber}, and $c'_1$ is a constant.
	
	\item 	Second, we estimate $EI_2$. Since $\mathcal{E}_{D}\left(\pi_L F_{D} \right) \leq \mathcal{E}_{D}\left(F_{D} \right)$, we have
	\begin{equation}\label{EI2}
		\begin{aligned}
			E{I_2}&\leq 2E\left\{\mathcal{E}_{D}\left(F_{D} \right)-\mathcal{E}_D\left(F_\rho \right)\right\}\\
			&= 2E\left\{\inf_{F\in \{H\circ L_K: H\in \mathcal{H}\}} \mathcal{E}_{D}(F)-\mathcal{E}_{D}(F_{\rho})\right\} \\
			& \leq 2\inf_{F\in \{H\circ L_K: H\in \mathcal{H}\}} E \left\{\mathcal{E}_{D}(F)-\mathcal{E}_{D}(F_{\rho})\right\} \\
			&= 2\inf_{F\in \{H\circ L_K: H\in \mathcal{H}\}} \|F-F_{\rho}\|_{\rho}^2.
		\end{aligned}
	\end{equation}
	
	\item Third, we estimate $EI_3$. Note that 
	\begin{equation*}
		\begin{aligned}
			EI_3&= E\left\{16L\sup_{H\in \mathcal{H}}\frac{1}{m}\sum_{i=1}^{m}\left|H(u_i)-H(\hat{u}_i)\right|\right\}\\
			&\le E\left\{16L\frac{1}{m}\sum_{i=1}^{m}\sup_{H\in \mathcal{H}}\left|H(u_i)-H(\hat{u}_i)\right|\right\}\\
			&\le 16L\cdot E\left\{\sup_{H\in \mathcal{H}}\left|H(u_1)-H(\hat{u}_1)\right|\right\}.
		\end{aligned}
	\end{equation*}
	According to Proposition \ref{raodong}, we have
	\begin{equation*}
		\sup_{H\in \mathcal{H}}\left|H(u_1)-H(\hat{u}_1)\right|  \le \frac{L}{\sqrt{C_\mu}C_K}N(N+1)^{d_1} (d_1^2+d_1) ||Tu_1-T\hat u_1||_1,
	\end{equation*}
	Note that 
	\begin{equation*}
		\begin{aligned}
			||Tu_1-T\hat u_1||_1&=\sum_{i=1}^{d_1}\left|\int u_1(t)\phi_i(t)d\mu(t)-\int \hat u_1(t)\phi_i(t)d\mu(t) \right|\\
			&\le d_1||u_1-\hat u_1||_{L_2(\mu)}.
		\end{aligned}
	\end{equation*}
	Therefore, we have
	\begin{equation*}
		\begin{aligned}
			EI_3 & \le \frac{16C_\mu L^2}{C_K}N(N+1)^{d_1} (d_1^3+d_1^2)  E\left(||u_1-\hat u_1||_{L_2(\mu)} \right) \\
			& \leq \frac{c'_5}{C_K} M^{1+\frac{1}{d_1}} E\left(||u_1-\hat u_1||_{L_2(\mu)} \right) ,
		\end{aligned}
	\end{equation*}
	where the last inequality is from (\ref{MNrelation}), and $c'_5= \frac{32L^2}{\sqrt{C_\mu} \bar C_1}$ is a positive constant. Since $K$ and $K_0$ are chosen as the cases stated in Section \ref{section2}, where  we have that $K(\cdot,t) \in \mathcal{H}_{K_0}$ for any $t \in \Omega$. Therefore, by Lemma \ref{bachresult} we get
	\begin{equation}
		E\left(||u_1-\hat u_1||_{L_2(\mu)} \right) \leq \sqrt{C_\mu C_{K_0}} E\left( ||u_1-\hat u_1||_{K_0} \right) \leq c'_6 \sqrt{\lambda_q},
	\end{equation}
	where $q= \frac{c'_5 n}{\log n}$, and $c'_5$, $c'_6= c'_4 \sqrt{C_\mu C_{K_0}}$ are positive constants. 
	Therefore, it follows that
	\begin{equation} \label{EI3}
		EI_3 \leq  \frac{c'_3}{C_K} M^{1+\frac{1}{d_1}} \sqrt{\lambda_q},
	\end{equation}
	where $c'_3=  c'_5 c'_6$.
\end{itemize}
Thus the proof is completed by combining \eqref{EI1}, \eqref{EI2}, and \eqref{EI3}.	\qed

\subsubsection{Proof of Theorem \ref{generalizationerror1}}

{\bf Proof}. 
By Theorem \ref{approxrate12}, there exists a functional network $F_{NN}$ with  $M$ nonzero parameters, first hidden layer width $d_1=\tilde{c}_1 \frac{\log M}{\log \log M}$, depth $J\le \tilde c_3 \left( \frac{\log M}{\log\log M}\right)^2 $, and embedding kernel being \eqref{embedkernel1} with $\gamma = \tilde{c}_2 \left(\frac{\log \log M}{\log M}\right)^{\frac{1}{d}} \log\log M \leq 1$, such that 
\begin{equation*}
	\sup_{f\in \mathcal{F}} |F_{NN}(f)-F_{\rho}(f)| \leq \tilde{c}_4 \left(\log M\right)^{-\frac{\alpha\lambda}{d}}  \left(\log\log 		M\right)^{\left(\frac{1}{d}+1\right) \alpha\lambda},
\end{equation*}
then we have
\begin{equation}
	\begin{aligned}
		\inf_{F\in \{H\circ L_K: H\in \mathcal{H}\}} \|F-F_{\rho}\|_{\rho}^2 & \le ||F_{NN}-F_{\rho}||_{\rho}^2
		\le \sup_{f\in \mathcal{F}}|F_{NN}(f)-F_{\rho}(f)|^2\\
		&\le \tilde c_4^2 \left(\log M\right)^{-\frac{2\alpha\lambda}{d}}  \left(\log\log M\right)^{2\left(\frac{1}{d}+1\right) \alpha\lambda}.
	\end{aligned}
\end{equation}
Plugging it into Theorem \ref{oracleinequality}, since $M^{\frac{1}{d_1}}= (\log M)^{\frac{1}{\tilde{c}_1}}$, $\lambda_q\leq \hat{C}_1 e^{-2c_2 q^{\frac{1}{d}}}$, and $C_K= \tilde{c}_0 \gamma^{-d}$ with $\tilde{c}_0$ being a positive constant, we get
\begin{equation}
	\begin{aligned}
		\mathop{E}||\pi_LF_{\widehat D}-F_{\rho}||^2_{\rho} \le & \frac{c'_1 \tilde c_3 M\log M\log m}{m}  \left( \frac{\log M}{\log\log M}\right)^2 \\
		&+ \frac{c'_3 c'_4}{\tilde{c}_0} \sqrt{\hat C_1} M (\log M)^{\frac{1}{\tilde{c}_1}}  e^{-c_2 (\frac{c'_5 n}{\log n})^{\frac{1}{d}}} \\
		& + 2\tilde c_4^2 \left(\log M\right)^{-\frac{2\alpha\lambda}{d}}  \left(\log\log M\right)^{2\left(\frac{1}{d}+1\right) \alpha\lambda}. \\
	\end{aligned}
\end{equation}
To balance the first and the third error terms, we can choose the number of nonzero parameters in the functional net as
\begin{equation}
	M = \left\lfloor \frac{m}{(\log m)^{\frac{2\alpha\lambda}{d}+4}} \right\rfloor.
\end{equation}
Then to make the second error term have the same rate as the other two terms, we can choose the second-stage sample size
\begin{equation}
	n \geq \hat C_2 (\log m)^d \log \log m,
\end{equation}
where $\hat C_2$ is a sufficiently large positive constant. When $m$ is sufficiently large, we have that $(\frac{2\alpha\lambda}{d}+4)\log \log m \leq \tilde{C}_1 \log m$ for some constant $\tilde{C}_1 \in (0,1)$. Therefore, we get the desired convergence rate with the constant $\tilde{C}_2= c'_1 \tilde c_3+ \frac{c'_3 c'_4}{\tilde{c}_0} \sqrt{\hat C_1}+ 2\tilde c_4^2 (1-\tilde{C}_1)^{-\frac{2\alpha\lambda}{d}} $.		\qed

\section{Useful Lemmas}
The following lemma, which demonstrates the rates of approximating a function $f$ by its embedded function $L_K f$ is utilized in our approximation analysis, and can be found in \cite[Theorem 2.2, Theorem 2.3]{Eberts2013}, where we choose $P_X$ as the Lebesgue measure on $\RR^d$, $p=\infty$, and $q=2$ as specified in their results.
\begin{lemma} \label{steinwart}
	Let $f \in L_2(\RR^d) \cap L_\infty(\RR^d)$, and suppose that the embedding kernel $K$ is chosen as \eqref{embedkernel1}.  Then we have
	\begin{equation}
		\|f- L_K f\|_{L_2(\RR^d)} \leq C_{d,r}  \omega_{r,L_2(\RR^d)}\left(f,\frac{\gamma}{\sqrt{2}} \right),
	\end{equation}
	where $C_{d,r}$ is a constant only depending on $d$ and $r$. Moreover, we have that $L_K f \in \mathcal{H}_{\gamma}(\RR^d)$, with the norm
	\begin{equation}
		\|L_K f\|_{H_\gamma} \leq (2^r-1)\pi^{-\frac{d}{4}} \|f\|_{L_2(\RR^d)} \gamma^{-\frac{d}{2}}.
	\end{equation}
\end{lemma}

Next, we present a result from our previous work \citep[Proposition 2]{songapproximation2023} that establishes the rates of approximating a continuous function by a deep ReLU neural network, which essentially leverages the idea of realizing the multivariate piecewise linear interpolation through a deep ReLU neural network, as discussed in \cite{yarotsky2018optimal}.

\begin{lemma} \label{lipschitzappro}
	Let $M,N \in \NN$, $\omega_g$ be the modulus of continuity of a function $g: [-R,R]^{d_1}\to [-L,L]$ with $L>0$, then there exists a deep ReLU neural network  $H_{NN}$ with depth $J= d_1^2+d_1+1$ and $M$ nonzero parameters, such that
	\begin{equation}
		\sup_{y\in [-R,R]^{d_1}} \left|g( y)-H_{NN}(y)\right| \leq 2 d_1 \omega_g\left(\frac{c_0 d_1^{\frac{4}{d_1}} R}{M^{\frac{1}{d_1}}} \right),
	\end{equation}
	where $c_0$ is a constant that is independent of $R, d_1$, and $M$. Moreover, $H_{NN}$ is constructed to output
	\begin{equation}
		H_{NN}(y)=\sum\limits_{i=1}^{(N+1)^{d_1}}c_i\psi\left(\frac{N}{2R}(y-b_i)\right),
	\end{equation}
	where $c_i\in \RR, b_i\in\mathbb{R}^{d_1}$ are free parameters that depend on $g$, and satisfy the conditions $|c_i| \leq L$ and $\|b_i\|_\infty \leq R$. The function $\psi: \RR^d \to \RR$ is defined as
	\begin{equation}\label{14}
		\psi(y)=\sigma\Big(\min\big\{\min_{k\neq j}(1+y_k-y_j),\min_{k}(1+y_k),\min_{k}(1-y_k)\big\}\Big).
	\end{equation}
	Furthermore, $N$ is the number of grid points in each direction, and has the following relationship with $M$:
	\begin{equation}
		\bar C_1 d_1^4 (N+1)^{d_1} \leq M \leq \bar C_2 d_1^4 (N+1)^{d_1},
	\end{equation}
	where $\bar C_1$ and $\bar C_2$ are positive constants.
\end{lemma}

The following lemma gives a bound of the difference between two $\min$ functions.
\begin{lemma}\label{lemma_min}
	For any $k\in\mathbb{N}$, and any $x_i,y_i\in\mathbb{R}$, $i=1,...,k$. Then there holds
	\begin{equation}
		\left|\min\{x_1,...,x_k\}-\min\{y_1,...,y_k\}\right|\le  2\sum_{i=1}^k|x_i-y_i|.
	\end{equation}
\end{lemma}

\noindent {\bf Proof}.
Note that for any $l\in\mathbb{N}$, we have $\min\{x_1,...,x_l\}=x_l-\sigma(x_l-\min\{x_1,...,x_{l-1}\})$, hence,
\begin{equation*}
	\begin{aligned}
		&\left|\min\{x_1,...,x_k\}-\min\{y_1,...,y_k\}\right|\\
		&\le \left| x_k-\sigma(x_k-\min\{x_1,...,x_{k-1}\})-y_k+\sigma(y_k-\min\{y_1,...,y_{k-1}\})\right|\\
		&\le 2|x_k-y_k|+	\left|\min\{x_1,...,x_{k-1}\}-\min\{y_1,...,y_{k-1}\}\right|\\
		&\le \cdots\\
		&\le 2|x_k-y_k|+...+2|x_2+y_2|+|x_1-y_1|\\
		&\le 2\sum_{i=1}^k|x_i-y_i|,
	\end{aligned}
\end{equation*}
which completes the proof.	

The following concentration inequality, which is employed in our generalization analysis, can be found in \cite[Theorem 11.4]{gyorfi2002distribution}. While this result specifically considers elements of $\mathcal{F}$ defined on $\mathbb{R}^d$, it is also applicable to situations where the elements of  $\mathcal{F}$ are defined on an arbitrary set.
\begin{lemma}\label{lemma_theorem9.14}
	Let $m\in\mathbb{N}$, and assume that $\rho(\{(f,y)\in \mathcal{Z}:|y|\le L\})=1$  for some $ L\ge 1$. Let $\mathcal{G}$ be a set of functions mapping from $\mathcal{F}$ to $[-L,L]$. Then for any $0<\delta \le 1/2$ and $\alpha,\beta>0$, we have
	\begin{equation}
		\begin{aligned}
			&P\left\{\exists G\in \mathcal{G}: ||G-F_{\rho}||^2_{\rho}-\left(\mathcal{E}_{D}(G)-\mathcal{E}_{D}(F_{\rho})\right)\ge\delta \left(\alpha+\beta+||G-F_{\rho}||^2_{\rho}\right)\right\}\\
			&\le 14\sup_{f_1^m}\mathcal{N}_1\left(\frac{\beta\delta}{20L},\mathcal{G},f_1^m\right)\exp\left(-\frac{\delta^2(1-\delta)\alpha m}{214(1+\delta)L^4}\right).
		\end{aligned}
	\end{equation}
\end{lemma}

The following lemma establishes the covering number bound for our functional network and is utilized in the generalization analysis.
\begin{lemma}\label{coveringnumber}
	Let $\epsilon>0$, then for any $m\in\mathbb{N}$ and $f_1,...,f_m\in \mathcal{F}$, we have 
	\begin{equation}
		\mathcal{N}_1\left(\epsilon,\pi_L\mathcal{H}_{d_1,M}\circ L_K,f_1^m\right)\le 2\left(\frac{4eL}{\epsilon}\right)^{c_2'JM\log M},
	\end{equation}
	where $J=d_1^2+d_1+1$, and $M$ is the number of nonzero parameters in the hypothesis space.
\end{lemma}

\noindent
{\bf Proof}. 
For any functional class $\mathcal{H}$, \cite[Lemma 9.2]{gyorfi2002distribution} shows a relationship between the $\epsilon$-covering number and the $\epsilon$-packing number,
\begin{equation}\label{lemma21}
	\mathcal{N}_1\left(\epsilon,\mathcal{H},x_1^m\right)
	\le \mathcal{M}_1\left(\epsilon,\mathcal{H},x_1^m\right).
\end{equation}
Furthermore, we denote $Pdim (\mathcal{H})$ as the pseudo-dimension of $\mathcal{H}$, which is defined as the largest integer $N$ for which there exists a set of points $(x_1,\dots,x_N,y_1,\dots,y_N)\in \mathcal{X}^N \times \RR^N$ such that for any binary vector $(a_1,\dots,a_N)\in \{0,1\}^N$, there exists some function $h \in \mathcal{H}$ that satisfies the condition
\begin{equation*}
	h(x_i)> y_i \Longleftrightarrow a_i=1, \qquad \forall i.
\end{equation*}
A relationship between the $\epsilon$-packing number and the pseudo-dimension was stated in \cite[Theorem 6]{haussler1992decision}, showing that
\begin{equation} \label{lemma22}
	\mathcal{M}_1\left(\epsilon,\pi_L \mathcal{H},x_1^m\right) \le \mathcal{M}_1\left(\epsilon, \mathcal{H},x_1^m\right) \leq 2\left(\frac{2eL}{\epsilon} \log \frac{2eL}{\epsilon} \right)^{Pdim(\mathcal{H})}.
\end{equation}
Then, by utilizing the pseudo-dimension bound for deep neural networks with any piecewise polynomial activation function, including ReLU as a specific case, as established in \cite[Theorem 7]{bartlett2019nearly}, we obtain a complexity bound for $\mathcal{H}_{NN}$, the class of deep ReLU networks $H_{NN}$ within the hypothesis space defined in (\ref{defhypospace}):
\begin{equation}\label{lemma23}
	Pdim( \mathcal{H}_{NN})	\le c_2' JM\log M,
\end{equation}
where $J=d_1^2+d_1+1$, and $c_2'$ is an absolute constant. Finally, it is important to note that the nonzero parameters in $\pi_L\mathcal{H}_{d_1,M}\circ L_K$ are exclusively contained within $\mathcal{H}_{NN}$. By combining (\ref{lemma21}), (\ref{lemma22}), and (\ref{lemma23}), we can derive the following result
$$\mathcal{N}_1\left(\epsilon,\pi_L\mathcal{H}_{d_1,M}\circ L_K,f_1^m\right)= \mathcal{N}_1\left(\epsilon,\pi_L\mathcal{H}_{NN},x_1^m\right) \le 2\left(\frac{4eL}{\epsilon}\right)^{c_2'JM\log M},$$
where $x_i= T (L_K f_i)$, for $i=1,\dots,m$. We thus complete the proof.	

The following lemma states the convergence rates of kernel quadrature when utilizing the theoretically optimal quadrature scheme outlined in Section \ref{section1}, as detailed in \cite[pp. 17--19]{bach2017equivalence}. In general, these convergence rates are influenced by the smoothness of the RKHS reproduced by the kernel and can be quantified by the eigenvalue decay of its associated integral operator.
\begin{lemma} \label{bachresult}
	Let $\Omega= [0,1]^d$, and let $K_0: \Omega \times \Omega \to \RR$ be a Mercer kernel with the eigensystem of its associated integral operator represented by $\{\lambda_i,\phi_i\}$. Additionally, let $K: \Omega \times \Omega \to \RR$ be a continuous kernel that satisfies $K(\cdot,t) \in \mathcal{H}_{K_0}$ for all $t \in \Omega$. For the kernel embedding step expressed in \eqref{kernelquadrature} and its theoretically optimal quadrature scheme detailed in \eqref{optimalquadrature}, we can establish the following error bound
	\begin{equation}
		E \left( \left\| L_K f_i - \widehat{L_K} f _i\right\|_{K_0} \right) \leq c'_4 \sqrt{\lambda_q},
	\end{equation}
	where $q= \frac{c'_5 n}{\log n}$, and $c'_4, c'_5$ are positive constants.
\end{lemma}

\vskip 0.2in

\bibliographystyle{plain}
\bibliography{Bib_functional_kernel}

\end{document}